\documentclass[10pt,twocolumn,letterpaper]{article}
\usepackage{cvpr}
 \cvprfinalcopy
\usepackage{times}

\usepackage{graphicx}
\usepackage{amsmath}
\usepackage{amssymb}
\usepackage{caption}
\usepackage{subcaption}
\usepackage[utf8]{inputenc} % allow utf-8 input
\usepackage[T1]{fontenc}    % use 8-bit T1 fonts
\usepackage{url}            % simple URL typesetting
\usepackage{booktabs}       % professional-quality tables
\usepackage{amsfonts}       % blackboard math symbols
\usepackage{nicefrac}       % compact symbols for 1/2, etc.
\usepackage{microtype}      % microtypography
\usepackage{algorithm} %ctan.org\pkg\algorithms
\usepackage{algpseudocode}
\usepackage{mathtools}

\usepackage{enumitem}
\usepackage[pagebackref=true,breaklinks=true,letterpaper=true,colorlinks,bookmarks=false]{hyperref}
\usepackage[usenames]{color}
\newtheorem{theorem}{Theorem}

\newtheorem{lemma}{Lemma}
\newtheorem{proposition}{Proposition}
\newtheorem{remark}{Remark}

\newtheorem{proof}{Proof}

\providecommand{\nor}[1]{\left\lVert {#1} \right\rVert}

\providecommand{\scalT}[2]{\left\langle{#1},{#2}\right\rangle}

\def\argmax{\operatornamewithlimits{arg\,max}}
\def\argmin{\operatornamewithlimits{arg\,min}}
\def\cvprPaperID{****} % *** Enter the CVPR Paper ID here

\ifcvprfinal\fi

\title{ Asymmetric Regularized CCA and Hierarchical Kernel Sentence Embedding for Image \& Text Retrieval   }

\author{Youssef Mroueh, Etienne Marcheret, Vaibhava Goel\\
IBM Watson Multimodal Group\\
{\tt\small mroueh,etiennem,vgoel@us.ibm.com}
}
\begin{document}

\maketitle

\begin{abstract} 
Joint modeling of language and vision has been drawing increasing
interest. A multimodal data representation allowing for
bidirectional retrieval of images by sentences and vice versa is a key
aspect. In this paper we present three contributions in  canonical correlation analysis (CCA) based multimodal retrieval. Firstly, we show that an asymmetric weighting of the canonical weights, while achieving a cross-view mapping from the search to the query space, improves the retrieval performance.  Secondly, we devise a computationally efficient model selection, crucial to generalization and stability, in the framework of the  \emph{Bj{\"o}rk Golub} algorithm for  regularized CCA  via spectral filtering.  Finally, we introduce a  Hierarchical Kernel Sentence Embedding (HKSE) that approximates Kernel CCA for a special similarity kernel  between distribution of words embedded in a vector space. State of the art results are obtained on MSCOCO and Flickr benchmarks when these three techniques are used in conjunction. 
%In this paper we show that a cross-view  mapping of the search space to the query space achieves state of the art performance in bidirectional retrieval using off the shelf features.   We show that this search to query mapping  can be implemented by a simple asymmetric weighting of the canonical correlation weights (AW-CCA), where the canonical weights of the search space are weighted by the canonical correlations. We revisit regularization for Canonical Correlation Analysis (CCA), devising a fast cross validation for CCA, within the framework of spectral filtering ( Tikhonov and truncated SVD (T-SVD) regularization) and  the Bjorck Golub algorithm. Finally we propose the randomized kernel mean embedding features as a means for aggregation of local features in text and image descriptions. 
%We present our regularized CCA algorithms within two spectral filtering regularization families:  Tikhonov regularization  and truncated SVD (T-SVD) regularization . T-SVD regularization is new to the best of our knowledge in the context of CCA and is faster to cross-validate than Tikhonov regularized CCA at the price of a small loss in accuracy.
%and  devise a fast method to choose the regularization parameter in the Tikhonov regularized CCA, using the regularization path of  T-SVD CCA and the singular values of the data matrices.   
\end{abstract} 
%\vspace{-0.35in} 

\section{Introduction: Multimodal Retrieval}
Modeling jointly language and vision has attracted a lot of attention recently. Generative models such as deep recurrent networks for the language modeling, in conjunction with deep convolutional neural networks on the image side have shown remarkable success in the task of image captioning \cite{Karpathy,mao2014explain,vinyals2014show}. Image and Text retrieval has been the focus of many recent works
\cite{MSR,Klein,skipthoughts,Gong}. The main contributions of this paper are :% Given an image and sentence embedding the goal is to design a multimodal representation that captures the correlation between the two modalities, and allows for bidirectional search.

%\vspace{-0.005 in}
%\subsection{This paper: Contributions}
{ 1) \emph{Mapping search items to the query space, AW-CCA.}} In multimodal retrieval it is more common to embed the search space and the query space 
into a shared space (CCA in \cite{Klein}) or to map the query item to the search space (cross-view mapping in \cite{Socher_zero}). In this paper we empirically show that  mapping the search items to the query space outperforms those standard methods. We further show that this cross-view mapping from the search to the query space can be implemented by a simple asymmetric weighting of the canonical correlation weights (AW-CCA), where the canonical weights of the search space are weighted by the canonical correlations.
% Intuitively, keeping the query intact and mapping the search items to the query space reduces the noise introduced on the query of interest. 
%2) \textbf{\emph{Asymmetric  and task dependent Correlation Weighting for CCA.}} 

{ 2.a) \emph{Regularization and Spectral Filtering for CCA.} }Regularization is a key factor for the numerical stability as well as the generalization properties of a learning algorithm. We revisit Regularized CCA \cite{RegularizedCCA} within the framework of spectral filtering and the  Bj{\"o}rk Golub Algorithm \cite{Golub73numericalmethods} (Algorithm \ref{ALG:BjorckGolub} in Appendix \ref{ap:algorithms}). We present our regularized CCA within two spectral filtering regularization families:  Tikhonov regularization \cite{RegularizedCCA}  and truncated SVD (T-SVD) regularization \cite{T-SVD}. T-SVD regularization is new in the CCA context.% While Tikhonov regularization is more popular for regularized CCA \cite{RegularizedCCA}, we introduce regularization for CCA using truncated SVD covariances. 
 
%Truncated SVD (T-SVD) is a popular regularizer for least squares regression problems \cite{T-SVD}, but to our knowledge has not been utilized before in the context of CCA .
 { 2.b) \emph{Fast T-SVD guided Tikhonov cross-validation for CCA.}}  We show that the truncated SVD CCA cross-validation can be computed more efficiently than the Tikhonov regularization path, at the price of a small loss in accuracy. 
In light of the spectral filtering interpretation we propose a hybrid algorithm that takes advantage of the fast computation of the regularization path of T-SVD CCA, in choosing a regularization parameter for the Tikhonov counterpart (Algorithm \ref{ALG:GTikhonovcca} in Appendix \ref{ap:algorithms}), enabling a computationally lightweight exhaustive model selection, thanks to this hybrid strategy.

{3) \emph{Hierarchical Kernel Sentence Embedding and State of the art results.}} We propose the Hierarchical Kernel Sentence Embedding (HKSE) as a means for aggregation of words embeddings (word2vec), that explains and outperforms the mean word2vec baseline . Using those features and mapping the search items to the query space via the asymmetric weighting of the Regularized CCA (T-SVD guided Tikhonov) and the cosine similarity, we achieve  state of the art bidirectional retrieval results  on the MSCOCO  \cite{MSCOCO} and Flickr benchmarks \cite{Flickr8k,Flickr30k} , with off the shelf features for image and text descriptions.\\
\noindent \textbf{Notation.}  $\mathcal{Q}$ and $\mathcal{S}$ are query and search spaces. Given a multimodal training set $S=\{(x_i,y_i)| x_i \in \mathcal{X}\subset \mathbb{R}^{m_x}, y_i \in \mathcal{Y}\subset \mathbb{R}^{m_y}, i=1\dots n\}$, ($n>\max(m_x,m_y)$), let $X \in \mathbb{R}^{n\times m_x}$, and $Y\in \mathbb{R}^{n\times m_y}$ be the two data matrices corresponding to each modality. Define $\mu_{X},\mu_{Y}$ to be the means of $X$ and $Y$ respectively. Let $C_{XX}= X^{\top}X-\mu_{X}\mu_{X}^{\top} \in \mathbb{R}^{m_x\times m_x}$, and $C_{YY}= Y^{\top}Y-\mu_{Y}\mu_{Y}^{\top} \in \mathbb{R}^{m_y\times m_y}$ be the covariances matrices of $X$ and $Y$ respectively. Let $C_{XY}=X^{\top}Y-\mu_{X}\mu_{Y}^{\top} \in \mathbb{R}^{m_x\times m_y}$ be the correlation matrix. Define $I_k$ to be the identity matrix in $k$ dimensions. SVD stands for the \emph{thin} singular value decomposition. A validation set $S_{v}$ is given for model selection. Our goal is to index a test set $S^*=\{(x^*_i,y^*_i)|  x^*_i \in S^*_{x} \subset \mathcal{X} , y^*_i \in S^*_{y} \subset \mathcal{Y}, i=1\dots n^*\}$ for bidirectional search. $X$ and $Y$ are assumed to be centered.  For $X$ non singular, let $X= U\Sigma V^{\top}$, and $C_{XX}=X^{\top}X$, we define $C^{-\frac{1}{2}}_{XX}=V\Sigma^{-1}$. $\sigma_{x,1}\dots \sigma_{x,m_x}$ are singular values in decreasing order.

\section{CCA and Image/Captions Bidirectional Retrieval }\label{sec:cca}
\noindent \textbf{Bidirectional Retrieval.} We start by defining more formally the bidirectional retrieval tasks. Given pairs of high dimensional points $(x_i,y_i)\in \mathcal{X}\times \mathcal{Y}$ where $x_i$ corresponds to the feature representation of an image given by a deep convolutional neural network, and $y_i$ a sentence embedding of an associated caption. Our goal is to index this multimodal data in a way that enables bidirectional retrieval: the image annotation task associating  a caption to a query  image and the image search task  associating an image to a query caption.\\

\noindent \textbf{Canonical Correlation Analysis.} We review in this section Canonical Correlation Analysis  due to \cite{hotelling1936relations}. For data matrices $X \in \mathbb{R}^{n\times m_x}$ and  $Y \in \mathbb{R}^{n\times m_y}$, let $k=\min(m_x,m_y)$ the canonical correlations $\sigma_1,\dots \sigma_k$ , and their corresponding pairs of correlations weights $\{(u_i,v_i)\}_{i=1\dots k}$, are given by the columns of $U\in \mathbb{R}^{m_x \times k}$ and $V \in \mathbb{R}^{m_y\times k}$, where $U$ and $V$ are the solution of the following maximization problem: 
$\max_{U^{\top}C_{XX}U=I_k , V^{\top}C_{YY}V=I_k} Tr(U^{\top}C_{XY}V),$
where $\sigma_i=u_i^{\top}C_{XY}v_i, i=1\dots k.$ We note $\Sigma$ the diagonal matrix having $\sigma_i$ on its diagonal. Intuitively CCA finds the directions that are maximally correlated and that are orthonormal in the metric defined by each covariance matrix, respectively. Bj{\"o}rk and Golub showed that the CCA problem is equivalent to the following formulation, minimizing the square of the Frobenius norm of the embeddings in the shared space:
\begin{equation}
\min_{U^{\top}C_{XX}U=I_k , V^{\top}C_{YY}V=I_k} \nor{XU-YV}^{2}_{F}.
\label{eq:ccajoint}
\end{equation}

\noindent \textbf{Cross-View Maps.} Using the linear maps $(U,V)$ from CCA, cross-view mapping of the multimodal data  is defined as mapping the first modality to  the second modality represented in the shared space. Interestingly we will see that those problems have a simple closed form solution using CCA. First for mapping images to captions embedded  in the shared space we solve:
\begin{equation}
\min_{W_x\in \mathbb{R}^{m_x\times k}}\nor{XW_x-YV}^2_{F} 
\label{eq:ccaTask1}
\end{equation}
the minimum is achieved by (assuming $X$ is non singular): 
\begin{eqnarray*}
W_x&=& (X^{\top}X)^{-1}X^{\top}YV
= (X^{\top}X)^{-1} (X^{\top}X)U\Sigma\\
&=& U \Sigma,
\end{eqnarray*}
where we used the property that CCA satisfies  as a generalized eigenvalue problem \cite{hotelling1936relations} i.e $X^{\top}Y V =(X^{\top}X) U\Sigma$.\\
Second for mapping captions to the images embedded in the shared space we solve:
\begin{equation}
\min_{W_y\in \mathbb{R}^{m_y\times k}}\nor{YW_y-XU}^2_{F} 
\label{eq:ccaTask2}
\end{equation}
similarly the minimum is achieved by $W_y= V\Sigma$. 

Those properties of cross-view mappings appeared for the first time in \cite{Golub1995}, for solving CCA in an alternating minimization fashion, we show in the following, that those properties are at the core of the bidirectional retrieval.
Interestingly we see from solutions of Equations \eqref{eq:ccaTask1} and \eqref{eq:ccaTask2}, that cross view maps between modalities and the shared space have a very simple form, suggesting embeddings having an asymmetric weighting of the canonical weights by the canonical correlations $(\Sigma U^{\top},V^{\top})$ and $(U^{\top},\Sigma V^{\top})$.
We answer in this Section the following questions: 
\emph{For a particular task which embedding should we use, the shared space embeddings $(U^{\top},V^{\top})$, or cross view mappings induced the by the asymmetric weighting? 
Given query and search spaces should the asymmetric weighting  be on the query space or  the search space canonical weight?}

We find empirically as discussed in Section \ref{sec:exp} that asymmetric weighting should be on the search space, and not on the query space canonical weights. Moreover this particular asymmetric weighting outperforms unweighted CCA, as well as a symmetric weighting of CCA introduced in \cite{Gong}. We give here an explanation of this empirical observation.\\

\noindent \textbf{Shared Space Versus Cross-view Maps.} To answer the first question, it is clear from the CCA formulation in Equation \eqref{eq:ccajoint}, that $(U,V)$ are optimized jointly for both retrieval tasks, hence embeddings $(U^{\top},V^{\top})$ are not optimized to their best for each task aside. On the other hand cross view maps in Equations \eqref{eq:ccaTask1} and \eqref{eq:ccaTask2} benefit from the shared space, and optimize for each task aside, hence are superior in bidirectional retrieval.% as we will confirm latter experimentally \textcolor{blue}{ (See Table \ref{tab:mappings} and Appendix \ref{app:Opt}).}

\noindent \textbf{Query Generation Versus  Search Generation.} To answer the second question, this goes back to a fundamental problem in Information Retrieval.
Given a query  $q \in \mathcal{Q}$ and a search item $s \in \mathcal{S}$ (referred to usually as a document), two probabilistic retrieval approaches are possible \cite{Relevance}: 1) The Search Generation approach, modeling $\mathbb{P}(s|q)$ (\emph{how likely is a search item  given the query}) known as the traditional probabilistic approach  2)  The query generation approach, modeling $\mathbb{P}(q|s)$ \emph{how likely is a query item given a search item}), known also as the language modeling approach. Lafferty et al \cite{Relevance} showed that while  both approaches are equivalent probabilistically as they are based on a different parametrization of a same joint \emph{relevance} likelihood, they are different statistically as the models are estimated differently. 
%\cite{Relevance} showed that, whereas the Query Generation approach implicitly models  the \emph{relevance} between the query and the search item, the Search generation approach needs an explicit modeling of the relevance by means of positive and negative examples (See Appendix \ref{ap:Laff}). The implicit relevance modeling makes the query generation approach appealing and thus widely used in language modeling.
 \cite{Relevance} propose to use a binary random variable $r$ that denotes relevance between a query and a search item, $r=1$ if there is a match and $0$ otherwise. In order to rank search items  \cite{Relevance} propose the  use  of the log-odds ratio:
 $\log \frac{\mathbb{P}(r=1|q,s)}{\mathbb{P}(r=0|q,s)}.$
Using the search generation approach this ratio is equivalent to \cite{Relevance}:
$\log \frac{\mathbb{P}(s|q,r=1)}{\mathbb{P}(s|q,r=0)},$
hence this approach needs both positives and negatives pairs of search and query items (triplet losses are instances of this approach).
On the other hand using the query generation approach (under mild assumptions) this ratio is equivalent to \cite{Relevance}:
$\log \mathbb{P}(q|s,r=1).$

As shown in {\cite{Relevance}} the query generation approach models relevance in an implicit way and does not need negative samples. %Note $U$ Adopting the query generation approach we obtain, 
The implicit relevance modeling makes the query generation approach appealing and we adopt it in the rest of this paper.
Consider for instance the image search task (text query, image search space) considering a gaussian model in the query generation approach: $q=V^{\top}y \sim \mathcal{N}(W_x^{\top} x,\sigma^2 I_{k}) , \sigma >0,s =x.$
Maximizing the log likelihood of this query generation approach corresponds to Equation \eqref{eq:ccaTask1} and  the optimal $W_x$ is given by $U\Sigma$. 
Hence for the image search task we can use the embeddings $(\Sigma U^{\top},V^{\top})$,  and the image search problem reduces to finding for a query caption $y^* \in S^*_{y}$, an image $x^*\in S^*_x$ using $\ell_2$ or cosine similarity, i.e $\argmin_{x^* \in S^*_{x}}  \nor{\Sigma U^{\top}x^*-V^{\top}y^*}^2_{2} $ or,
\vspace{-0.05 in}
\begin{equation}
 \argmax_{x^* \in S^*_{x}} \frac{\scalT{\Sigma U^{\top}x^*}{ V^{\top} y^*}}{\nor{\Sigma U^{\top}x^*}_2\nor{V^{\top} y^*}_2}.
\label{eq:imSearchcca1}
\vspace{-0.1 in}
\end{equation} 

Similarly for the image annotation problem we can use the embeddings $(U^{\top},\Sigma V^{\top})$, and it reduces  also to finding for a query image $x^* \in S^*_{x}$, and a caption $y^* \in S^*_{y}$ using $\ell_2$ or cosine similarity, i.e $\argmin_{y^* \in S^*_{y}} \nor{U^{\top}x^*-\Sigma V^{\top}y^*}^2_2$ or,
\vspace{-0.05in}
\begin{equation}
 \argmax_{y^* \in S^*_{y}} \frac{\scalT{ U^{\top}x^*}{ \Sigma V^{\top} y^*}}{\nor{ U^{\top}x^*}_2\nor{\Sigma V^{\top} y^*}_2}.
\label{eq:imAnnotationcca1}
\vspace{-0.1 in}
\end{equation} 
Thus using the query generation approach we see that $\Sigma$ appears in an asymmetric way in the embedding of the points depending on the task: the canonical weight of the search space is weighted by the canonical correlations. Hence we call our method asymmetrically weighted CCA (Table \ref{tab:asymW})%. Asymmetrically weighted CCA (Table \ref{tab:asymW}) performs approximately\footnote{Using thin SVD $P_{y}$ is not a complete  orthonormal basis  of $\mathbb{R}^{m_{y}}.$} the optimal nearest neighbor search in the least square sense for each task. %\subsection{ Asymmetric Weighting: Important particular powers choices}
%\vspace{-0.05in}

\begin{table}[h]
%\hspace{-0.5in}
\label{sample-table}
\begin{center}
\begin{tabular}{lll}
\multicolumn{1}{c}{\bf Task}  &\multicolumn{1}{c}{\bf Image Embedding} &\multicolumn{1}{c}{\bf Caption Embedding}
\\ \hline \\
Search       &$\Sigma U^{\top}x^*$&$ V^{\top}y^*$ \\
 Annotation          &$U^{\top}x^*$ &$\Sigma V^{\top}y^*$
\end{tabular}
\end{center}
\caption{ Asymmetrically Weighted CCA: Task dependent embeddings, in the query generation approach. $x^*$ is a test image, $y^*$ is a test caption. $(U,V)$ are the canonical weights of $X$ and $Y$. $\Sigma$ is the diagonal  canonical correlations matrix.}
\label{tab:asymW}
\vspace{-0.4 in}
\end{table}
\begin{remark}
1) We found in practice that the cosine similarity between the embeddings outperforms $\ell_2$ based retrieval. This is in line with findings in the CCA based retrieval where cosine similarities are used  \cite{Klein}, \cite{Gong} and reported to be outperforming $\ell_2$. 2) Intuitively, in the query generation approach, the canonical correlations weigh the search canonical weights directions. This weighting will favor search canonical directions that are highly correlated with the query. Hence every search item is strengthened and asked to match the query. In the search generation approach, the canonical correlations weigh the query canonical weights directions. The query will get strengthened alone and asked to match all search items, which hints to why this approach is weaker : distributing strength on all search items (query generation approach), is better than stressing it on the query alone (Search generation approach). 
\end{remark}

%\begin{table}[h]
%%\hspace{-0.0005in}
%\label{sample-table}
%\begin{center}
%\begin{tabular}{lll}
%\multicolumn{1}{c}{\bf Mapping}  &\multicolumn{1}{c}{\bf  Search R@10} &\multicolumn{1}{c}{\bf Annotation R@10}
%\\ \hline \\
%Search Generation     & ~~~~34.61~~~&  ~~~~35.74\\
%CCA & ~~~~38.63 & ~~~~47.16\\
%Query Generation \\ (this paper)&\textbf{~~~~44.24} & \textbf{~~~~53.12}
%\end{tabular}
%\end{center}
%\caption{Various mappings performance in bidirectional retrieval on the MSCOCO Benchmark($5 K$ test,VGG+Mean word2vec) (in \%).}
%\label{tab:mappings}
%\vspace{-0.23 in}
%\end{table}

\section{CCA Computation and Regularization}
\noindent In this section we focus on the computational  aspects of CCA and its regularization. CCA can be solved as a generalized eigenvalue problem or using singular value decomposition by virtue of the Bj{\"o}rk Golub algoritthm.  As we will see in this Section solving using SVD offers computational advantages, and allows for regularization using spectral filtering techniques. \\  
\noindent \textbf{Bj{\"o}rk Golub Algorithm.} The following Lemma due to Bj{\"o}rk and Golub shows that the canonical correlation weights can be computed  using  SVD: %the singular value decomposition of the data matrices $X$ and $Y$, and the correlation matrix in the whitened space :
\vspace{-0.05 in}
\begin{lemma}[\cite{Golub73numericalmethods}]
Let $X=U_x\Sigma_xV^{\top}_x$, and $Y=U_y\Sigma_yV^{\top}_y$ be the singular value decomposition of $X$ and $Y$ $(U_x \in \mathbb{R}^{n\times m_x},\Sigma_x\in \mathbb{R}^{m_x\times m_x},V_x\in \mathbb{R}^{m_x\times m_x})$. Let  $k=\min(m_x,m_y)$, and $T=U^{\top}_xU_y$. Let
%\begin{equation}
 $T=P_{x}\Sigma P^{\top}_y $ 
 %\label{eq:T-SVD}
 %\end{equation}
  be its SVD,  $P_x \in \mathbb{R}^{m_x\times k}, \Sigma \in \mathbb{R}^{k\times k}, P_{y} \in \mathbb{R}^{m_y\times k}$. The canonical correlations  of $X$ and $Y$ are the diagonal elements of $\Sigma$, with canonical weights of $X$ given by $U=V_{x}\Sigma^{-1}_{x}P_{x}$, and the canonical weights of $\,Y$ given by $V= V_{y}\Sigma^{-1}_{y}P_{y}$ . (Proof  in  Appendix \ref{ap:proofs}).
\label{lem:BG} 
\vspace{-0.05 in}
\end{lemma}

Algorithm \ref{ALG:BjorckGolub} (Appendix \ref{ap:algorithms}) summarizes the Bj{\"o}rk Golub procedure to compute CCA. % While the original algorithm of Bjorck and Golub uses the QR factorization of $X$ and $Y$, we follow \cite{Avron} in exposing the algorithm fully with SVD. Note that both $U$ and $V$ correspond to a whitening step followed by a projection to a common space of dimension $k$. 
The total computational complexity of this algorithm assuming $m_y<m_x$ is $O(nm^2_x+nm_y^2+m_xm_y^2 )$. 
%While the Bjorck Golub SVD algorithm is intuitive and efficient, it is less popular in  machine learning than the generalized eigenvalue implementation of CCA.\\
 For now we have assumed that the covariances $C_{XX}$ and $C_{YY}$ are non-singular, and we presented an SVD version of the Bj{\"o}rk Golub Algorithm in this context. Regularizing CCA does not only allow for numerical stability in the non singular case, it also allows for better generalization properties and avoids overfitting. \\
%Tikhonov regularization is the most common regularization used in CCA and consists in replacing the covariances by $C_{XX}+\gamma_x I_{m_x}$ and $C_{YY}+\gamma_y I_{m_y}$, where $\gamma_{x},\gamma_y>0$ are the regularization parameters subject to cross-validation. In this section we extend the SVD Bjorck Golub Algorithm to the Tikhnov regularized CCA.
%Another  contribution of this paper is in introducing the truncated SVD regularization to the covariances in the CCA problem. We emphasize  that our truncation is applied to the SVD of the  data matrices in the covariances of  $X$ and $Y$, not the SVD of  the whitened correlation matrix $T$ in the Bjorck Golub Algorithm.  Truncating the SVD of $T$ and  choosing an embedding dimension $k<\min(m_x,m_y)$  is sometimes referred to as the truncated SVD CCA \cite{tsvdT}, hence our clarification. %We replace the covariance $C_{XX}$ by its best $k_x$-rank approximation given by the SVD of $X$, and $C_{YY}$ by its best $k_y$-rank approximation given by the SVD of $Y$, where $k_x,k_y \in \mathbb{N}$, $k_{x}\leq m_x$, and $k_{y} \leq m_y$. In this framework $k_x$ and $k_y$ are the regularization parameters and are subject to cross-validation.
%We show in the following how to specialize the Bjorck Golub algorithm to handle T-SVD regularized CCA.
\noindent \textbf{Tikhonov and Truncated SVD Regularization.}\label{sec:BGTikhonov.}
The Tikhonov regularized CCA problem \cite{RegularizedCCA} for parameters $\gamma_x,\gamma_y >0$ can be written in this form:
\begin{equation}
\max_{U^{\top}(X^{\top}X+\gamma_x I_{m_x})U=I_{k} ,V^{\top}(Y^{\top}Y+\gamma_y I_{m_y})V=I_{k} }Tr(U^{\top}X^{\top}YV).
\label{eq:rtikh}
\end{equation}
Let $k_x\leq m_x$ and let $X_{k_x}$ be the best $k_x$-rank approximation of $X$ given by the truncated SVD: 
 $X_{k_x} = U_{k_x}\Sigma_{k_x}V_{k_x}^{\top}, ~~U_{k_x} \in \mathbb{R}^{n\times k_x}, \Sigma_{k_x}\in \mathbb{R}^{k_x\times k_x}, V_{k_x}\in \mathbb{R}^{m_x\times k_x}.$ Similarly for $Y$, we define the best $k_y$-rank approximation $(k_y\leq m_y)$: $Y_{k_y} = U_{k_y}\Sigma_{k_y}V_{k_y}^{\top}, ~~U_{k_y} \in \mathbb{R}^{n\times k_y}, \Sigma_{k_y}\in \mathbb{R}^{k_y\times k_y}, V_{k_y}\in \mathbb{R}^{m_y\times k_y}$.
We define the truncated SVD CCA as follows:
\begin{equation}
\max_{U^{\top}X^{\top}_{k_x}X_{k_x}U=I_{k}, V^{\top}Y_{k_y}^{\top}Y_{k_y}V=I_{k} }Tr(U^{\top}X^{\top}YV)
\label{eq:rT-SVD}
\end{equation}
The following Theorem Proved in Appendix B shows how the Bj{\"o}rk-Golub procedure to compute the canonical weights extends to the regularized case, using singular value decompositions of the data matrices and a correlation operator.
\vskip -0.2in
\begin{theorem}[Regularized CCA] Let $X=U_x\Sigma_xV^{\top}_x$, and $Y=U_y\Sigma_yV^{\top}_y$ be the singular value decomposition of $X$ and $Y$ $(U_x \in \mathbb{R}^{n\times m_x},\Sigma_x\in \mathbb{R}^{m_x\times m_x},V_x\in \mathbb{R}^{m_x\times m_x})$. 
 \begin{enumerate}[leftmargin=0.2in] 
\item {{Tikhonov Regularization.}} Let  $k=\min(m_x,m_y)$. Define the Tikhonov regularized correlation operator $$T_{\gamma_x,\gamma_y}= \left(\Sigma^2_x+\gamma_x I\right)^{-\frac{1}{2}}\Sigma_x \left(U^{\top}_xU_y\right)  \Sigma_y  \left(\Sigma^2_y+\gamma_y I\right)^{-\frac{1}{2}}, $$
and let $T_{\gamma_x,\gamma_y}= P_x\Sigma P^{\top}_{y}$ be its singular value decomposition $(P_{x}\in \mathbb{R}^{m_{x}\times k}, \Sigma \in \mathbb{R}^{k\times k}, P_{y}\in \mathbb{R}^{m_{y}\times k})$. 
The canonical weights of the Tikhonov regularized CCA  \eqref{eq:rtikh} are given by $U =V_x  \left(\Sigma^2_x+\gamma_x I\right)^{-\frac{1}{2}}P_x $ and $V =V_y  \left(\Sigma^2_y+\gamma_y I\right)^{-\frac{1}{2}}P_y$. Canonical correlations are given by $\Sigma$.
\item {{Truncated SVD Regularization.}} Define the T-SVD regularized correlation operator 
$T_{k_x,k_y}=U_{k_x}^{\top}U_{k_y},$
and let $T_{k_x,k_y}= P_x\Sigma P^{\top}_{y}$ be its singular value decomposition. 
The canonical weights of the T-SVD regularized CCA  \eqref{eq:rT-SVD} are given by $U =V_{k_x} \Sigma_{k_x}^{-1}P_{x}$ and $V =V_{k_y} \Sigma_{k_y}^{-1}P_{y}.$  Canonical correlations are given by $\Sigma$.
\item Spectral Filtering. (Tikhonov)  Define $f_{\gamma_x}$  a spectral filter acting on the singular values of $X$, such that $f_{\gamma_x,j}= \frac{\sigma_{x,j}}{\sqrt{\sigma^2_{x,j}+\gamma_x}},j=1\dots m_x$. Similarly define $f_{\gamma_y}$. we have:
\vskip -0.2 in
\begin{equation}
T_{\gamma_x,\gamma_y}=  f_{\gamma_x}(\Sigma_x) U_x^{\top}U_{y}f_{\gamma_y}(\Sigma_y).
\label{eq:SpFilterTikhonov}
\end{equation}
(T-SVD) Define $f_{k_x}$  an element wise filter acting on the singular values of $X$, such that :
$f_{k_x}(\sigma_{x,j})=0 \text{ if } \sigma_{x,j} <\sigma_{x,k_x},$
and  $f_{k_x}(\sigma_{x,j})=1 \text{ if } \sigma_{x,j} \geq \sigma_{x,k_x}.$ Let
%\begin{equation}
$T^{k_x,k_y}_f= f_{k_x}(\Sigma_x) U_x^{\top}U_{y} f_{k_y}(\Sigma_y),$
%\label{eq:SpFilterT-SVD}
%\end{equation}
For a matrix $A$, $A(1:k_x,1:k_y)$, refers to the sub-matrix containing the first $k_x$ rows and the first $k_y$ columns of $A$. We have:
%\vskip -0.2 in
\begin{equation}
T_{k_x,k_y}= T^{k_x,k_y}_f (1:k_x,1:k_y).
\label{eq:SpFilterT-SVD}
\end{equation}
\end{enumerate}
\label{theo:Reg}
\end{theorem}
%\begin{proof}
%The proof is given in %\end{proof}
%\begin{remark}
% ~
%\begin{enumerate}
%\item 
%1)  While $T_{k_x,k_y} \in \mathbb{R}^{k_x\times k_y}$ has a SVD computational cost of  $\min(O(k_xk_y^2) ,O(k_y k^2_x))$, $T_{\gamma_x,\gamma_y} \in \mathbb{R}^{m_x\times m_y}$ and has a SVD computational cost of  $\min(O(m_xm_y^2) ,O(m_y m^2_x))$. Hence T-SVD is more efficient computationally. 
%2) In T-SVD CCA the dimension of the embedding space is  $k=\min(k_x,k_y)$. In the Tikhonov case, it is $k=\min(m_x,m_y)$ (one can also compute a truncated SVD of $T_{\gamma_x,\gamma_y}$ to reduce further the dimensionality of the embedding).
%\end{enumerate}
%\end{remark}
\textbf{Spectral Filtering for CCA.}\label{sec:sfilter}
From Equations   \eqref{eq:SpFilterTikhonov} and \eqref{eq:SpFilterT-SVD}  we see that both regularizations are proceeding by filtering of singular values.
While T-SVD proceeds by a hard filtering, Tikhonov proceeds with a soft filtering. In order to see the correspondence between T-SVD and Tikhonov regularization it is important to consider the following choice of regularization parameters:
\emph{For a choice of $(k_x,k_y)$ in T-SVD, consider $\gamma_x= \sigma^2_{x,k_x}$and $\gamma_y=\sigma^2_{y,k_y}$ in Tikhonov Regularization.}
With this particular choice the spectral filters for  Tikhonov regularization  become:
%\vspace{-0.06 in}
%\begin{equation}
$f_{\sigma^2_{k_x}}(\sigma_{x,j})= \frac{\sigma_{x,j}}{\sqrt{\sigma_{x,j}^2+\sigma^2_{x,k_x}}}. $
%\label{eq:filterTikhonov}
%\end{equation}
%To appreciate this particular choice of $\gamma_x,\gamma_y$, it is important to compare the spectral filter of T-SVD regularization for CCA given in Equation \eqref{eq:filterT-SVD} and the spectral filter for  Tikhonov Regularization for CCA given in Equation \eqref{eq:filterTikhonov}.
Let $\alpha >0$, consider:
$1)~ g_{\text{hard}}(x)= 0 , \text{ if } 0\leq x<\alpha, \text{ and } g_{\text{hard}}(x)= 1  \text{ if } x \geq \alpha.$
$2)~ g_{\text{soft}}(x) = \frac{x}{\sqrt{x^2+\alpha^2}}.$
It is easy to see that Equations  \eqref{eq:SpFilterTikhonov} and \eqref{eq:SpFilterT-SVD}, correspond to the element-wise application of  $g_{\text{hard}}$ and $g_{\text{soft}}$ respectively for $\alpha=\sigma_{k_x}$( See Figure \ref{fig:sfilter} for $g_{\text{hard}}$ versus $g_{\text{soft}}$, for $\alpha=20$).  %We see  that both T-SVD and Tikhonov Regularization  correspond to a spectral filtering of the singular values of $X$ and $Y$. T-SVD is a hard pruning of directions corresponding to singular values less then the threshold defined by $\sigma_{k_x}$. For $\alpha=\sigma_{k_x}$ Tikhonov corresponds to a soft pruning of those directions. 
 \begin{figure}[h!]
 \vspace{-0.003 in}
  \centering
  \includegraphics[scale=0.23]{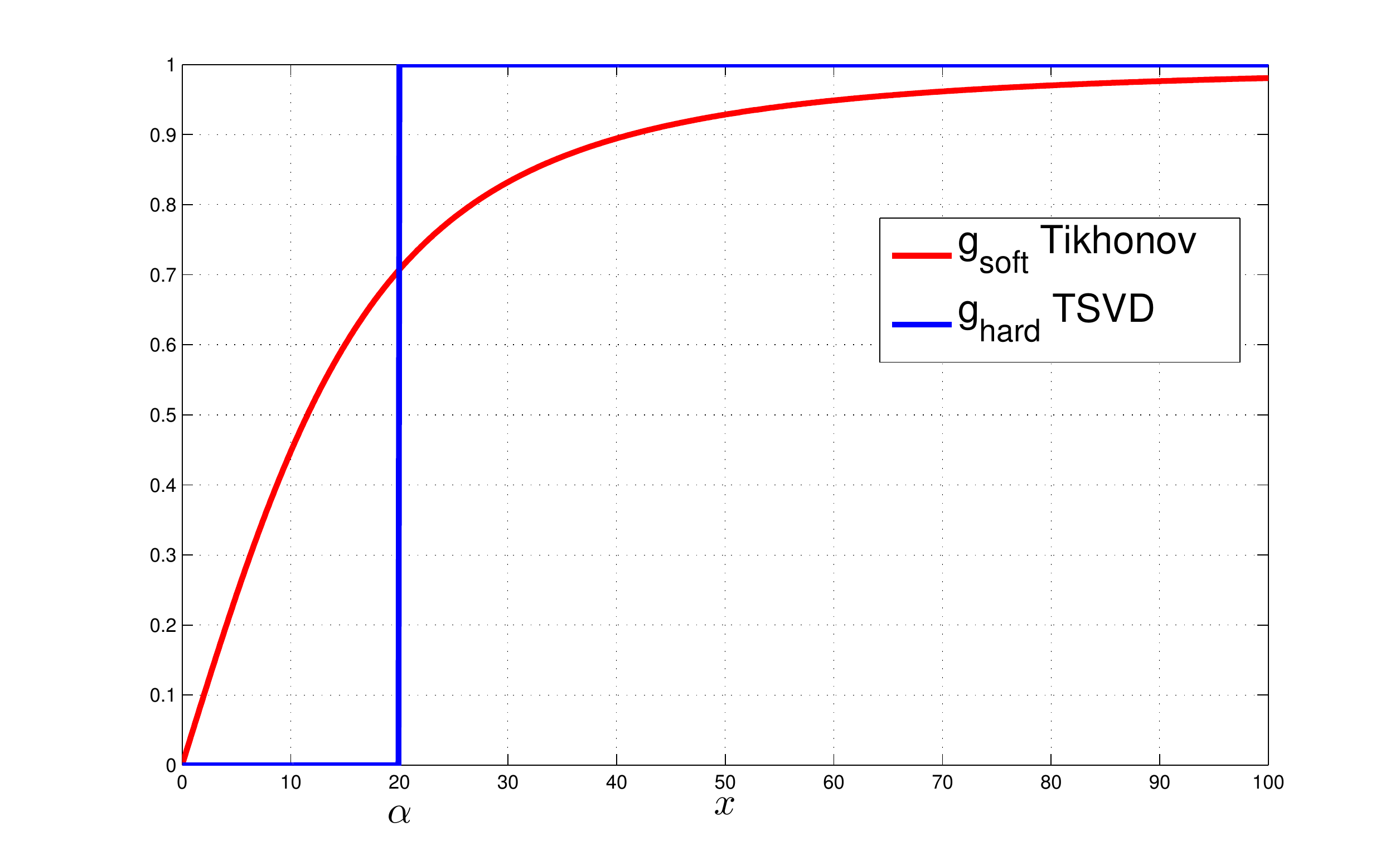}
  \caption{Both T-SVD and Tikhonov Regularization  correspond to a spectral filtering of the singular values of $X$ and $Y$. T-SVD is a hard pruning of directions corresponding to singular values less than the threshold defined by $\sigma_{k_x}$. For $\alpha=\sigma_{k_x}$ Tikhonov corresponds to a soft pruning of those directions.  }
  \label{fig:sfilter}
\end{figure}
\vspace{-0.1 in}
In conclusion, using this spectral filtering approach the natural set of regularization parameters for Tikhonov regularization $(\gamma_x,\gamma_y)$ is therefore the set of singular values squared of $X$ and $Y$ respectively. Hence in order to cross-validate CCA we propose the following three approaches:\\
1) \textbf{\emph{Tikhonov-CCA:}} Perform a grid search on $(\gamma_x,\gamma_y) \in \{\sigma^2_{1,x},\dots,\sigma^2_{m_x,x} \}\times \{\sigma^2_{1,y},\dots\sigma^2_{m_y,y} \}$, using the SVD of  $T_{\gamma_x,\gamma_y}$. Each SVD costs $\min(O(m_xm_y^2) ,O(m_y m^2_x))$ (Algorithm \ref{ALG:Tikhonovcca} in Appendix \ref{ap:algorithms}). \\
2) \textbf{\emph{T-SVD-CCA:}} Perform a grid search on $(k_x,k_y) \in [1,m_x]\times [1,m_y]$, using the SVD of   $T_{k_x,k_y}$. This grid search (Algorithm \ref{ALG:TruncatedSVDcca} in Appendix \ref{ap:algorithms}) is computationally efficient since each SVD  costs $\min(O(k_xk_y^2) ,O(k_y k^2_x))$, hence more efficient than Tikhonov. \\%(See Appendix \ref{app:Sf} Table \ref{tab:asymWTIMING} for CPU timing) .\\
3) \textbf{\emph{G-Tikhonov-CCA:}} While there is no exact one to one correspondence between the soft pruning and the hard pruning in T-SVD and Tikhonov, the spectral filtering interpretation suggests that the optimal $(k^*_x,k^*_y)$ from T-SVD cross-validation (computationally efficient), can be used as a proxy for the optimal Tikhonov cross validation (computationally expensive), by simply setting $(\gamma_x,\gamma_y)=(\sigma^2_{x,k^*_x},\sigma^2_{y,k^*_y}) $ with a small loss in the accuracy. This is summarized in Algorithm \ref{ALG:GTikhonovcca} in Appendix \ref{ap:algorithms}. Algorithm \ref{ALG:GTikhonovcca}, takes advantage of the fast computation of T-SVD regularization for CCA, and the spectral filtering interpretation in order to choose a good regularization parameter for Tikhonov Regularization.

%\begin{table}[h]
%\hspace*{-0.05 in}
%\centering
%\small\addtolength{\tabcolsep}{-3 pt}
%\begin{tabular}{{|l|llll|llll|l}}
%\hline
%& \multicolumn{4}{|c|}{Image search} & \multicolumn{4}{|c|}{Image annotation} \\
% AW-CCA& r@1 & r@5 & r@10 & med &  r@1 & r@5 & r@10 & med  \\
%& & & & rank  & & & & rank \\
%\hline
%Tikh &\textbf{23.08}&\textbf{50.62}&\textbf{63.53}&\textbf{5.0}&\textbf{30.66}&\textbf{59.1}&\textbf{70.13}&\textbf{3.0} \\                                                                  
%T-SVD  & 20.61&46.91&59.9&6.0&27.0&55.5&67.80&4.0  \\  
% G-Tikh &\textbf{22.40}&\textbf{49.94}&\textbf{62.62}&\textbf{6.0}&\textbf{30.56}&\textbf{58.43}&\textbf{69.96}&\textbf{4.0} \\  
% \hline                                                                
%\end{tabular}
%\vspace{.2cm}
%\caption{Mean results (in \%) of the test splits on the Flickr 30 K, in average w2vec/vgg setup. We see that G-Tikh and Tikh are on par.}% \textcolor{blue}{med r cv, we HAVE R@1 for full tikh not the rest}}
%\label{tab:flickrComp}
%\vspace{-0.2 in}
%\end{table} 

\section{Hierarchical Kernel Sentence Embedding}\label{sec:features}
%\vspace{0.01 in}
%\textbf{Image Features.} For feature extraction we follow \cite{Klein}, and use VGG CNN features with a linear Kernel.

\noindent \textbf{Features.} After proposing our asymmetric CCA and an efficient way to cross validate it for bidirectional retrieval of images and sentences, we address in this Section the question of feature representation for both modalities. Convolutional neural networks are the gold standard in computer vision, hence we focus here 
on defining   a new sentence  embedding for aggregating local features in text description.
 Given a vocabulary $\mathcal{A}$ represented by a vector space i.e a word embedding, word2vec for instance  \cite{word2vec}, a sentence can be seen as a distribution $\rho$ on $\mathcal{A}$.

\noindent \textbf{A Kernel between Distributions.} Given a distribution $\rho$ defined on a vocabulary space $\mathcal{A} \subset \mathbb{R}^{d}$. Let $k_{\gamma,d}$ be a shift invariant kernel such as the gaussian kernel with  parameter $\gamma$, for $a,b\in \mathcal{A}$, $k_{\gamma,d}(a,b)=\exp(-\frac{\gamma}{2}\nor{a-b}^2)$. Let $\mathcal{H}_{k}$ be the associated Reproducing Kernel Hilbert Space (RKHS), with norm $\nor{.}_{\mathcal{H}_{k}}$. The kernel mean embedding of $\rho$ is defined \cite{kmeanemb1} as follows $\mu(\rho) = \int_{\mathcal{A}}k_{\gamma,d}(a,.)\rho(a)da.$ $\mu \in \mathcal{H}_{k}$ and can be used to define a distance between two  distributions $\rho_1$ and $\rho_2$ by means of $\nor{\mu(\rho_1)-\mu(\rho_2)}^2_{\mathcal{H}_{k}}$.
Given a finite sample $\{a_1,\dots a_{n}\}$ from $\rho$ the empirical kernel mean embedding is given by:  
$\mu_{n}(\rho) = \frac{1}{n}\sum_{i=1}^n k_{\gamma,d}(a_i,.).$ 
Let $\eta>0$, define the following kernel \cite{kmeanemb1,BagWords1,UnivKernel} between  distributions  $\rho_1$ and $\rho_2$ defined on $\mathcal{A}$, given a finite sample $\{a_1\dots a_{n_1}\}$ and $\{b_1\dots b_{n_2}\}$ from each distribution respectively:\vspace{-0.06 in}
$$K(\rho_1,\rho_2)=  \exp \left( -\frac{\eta}{2}\nor{\mu_{n_1}(\rho_1)-\mu_{n_2}(\rho_2)}^2_{\mathcal{H}_{k}} \right). \vspace{-0.06 in}$$

It is easy to see that $K(\rho_1,\rho_2)=\exp \frac{\eta}{2}(\Delta)$, where $\Delta=\frac{1}{n_1n_2}\sum_{i=1}^{n_1}\sum_{j=1}^{n_2}2k_{\gamma,d}(a_i,b_j)-\frac{1}{n_1^2} \sum_{i,j=1}^{n_1}k_{\gamma,d}(a_i,a_j)-\frac{1}{n^2_2}\sum_{i,j=1}^{n_2}k_{\gamma,d}(b_i,b_j)$.\!\! A sentence with words  $a_i, i=1\dots n$ embedded in $\mathcal{A}$, can be seen as a set of finite samples from a distribution $\rho$. Hence to compare two sentences we can use the kernel $K$ defined above. Note that $K$ is the composition of a word level similarity  where a pairwise comparison (intra and inter sentences) of words is computed using a first level gaussian kernel, and a sentence level similarity on the average pairwise similarity of words using a second level gaussian kernel. \\% $K$ can be seen also as a kernel between sets or bag of words kernel \cite{kmeanemb1, BagWords1}. $K$ defines a universal kernel on distributions as shown in \cite{UnivKernel}. 

\noindent \textbf{Hierarchical Kernel Sentence  Embedding (HKSE)}. Rather then using the kernel $K$ and kernel CCA, that are computationally expensive, we make use of random Fourier features \cite{RF} in approximating $K$ with an explicit feature map i.e a sentence embedding. Let $\Phi_{\gamma,d}(a) \in \mathbb{R}^{m}, \Phi_{\gamma,d}(a)=\frac{\sqrt{2}}{\sqrt{m}}(\cos(\scalT{w_1}{a}+b_1)\dots \cos(\scalT{w_m}{a}+b_m)), w_j \sim \mathcal{N}(0,\gamma I_{d})$, and $b \sim \rm{Unif}[0,2\pi]$, we have $\scalT{\Phi_{\gamma,d}(a)}{\Phi_{\gamma,d}(b)}\approx k_{\gamma,d}(a,b)$. Hence we define the randomized kernel mean map of $\rho,\hat{\mu}_n(\rho)= \frac{1}{n}\sum_{i=1}^n \Phi_{\gamma,d}(a_i).$  For sufficiently large $m$ we have %\vspace{-0.06 in}
$$K(\rho_1,\rho_2)\approx  k_{\eta,m}\left( \hat{\mu}_{n_1}(\rho_1) ,\hat{\mu}_{n_2}(\rho_2) \right).$$
$k_{\eta,m}$ is also a shift invariant kernel in $m$ dimension, that can be in its turn approximated with a random feature map $\Phi_{\eta,m} \in \mathbb{R}^{m'}$. For sufficiently large $m'$, we have $$K(\rho_1,\rho_2)\approx \scalT{\Phi_{\eta,m}\left(\hat{\mu}_{n_1}(\rho_1)\right) }{\Phi_{\eta,m}\left( \hat{\mu}_{n_2}(\rho_2)\right)}_{\mathbb{R}^{m'}}.$$
Hence $K$ is approximated by a deep ($2$ layers) map, that is the composition of a non linear average pooling and a non linear feature map. 
By embedding sentences consisting of word vectors (word2vec) ${a_1\dots a_n}$  using $\Phi_{\eta,m}\left(\frac{1}{n}\sum_{i=1}^{n} \Phi_{\gamma,d}(a_i)\right)$, we compute implicitly the similarity $K$ between the bag of words distributions. While $k_{\gamma,d}$ or $\Phi_{\gamma,d} $ act as a localizer at the word level,  $k_{\eta,m}$ or $\Phi_{\eta,m}$ localize on the sentence level. Hence we note them respectively with $k^w,\Phi^w$ (word level kernel), and $k^s,\Phi^s$ (sentence level kernel).  We note the  embedding with HKSE($k^w,k^s$). The average word2vec representation proposed in \cite{Klein} corresponds to HKSE($k^w,k^s$), where $k^w$ and $k^s$ are linear kernels. 
%It is easy to see that the inner dimension $m$ scales logarithmically with the size of the vocabulary, the outer dimension $m'$ scales linearly with the maximum sentence length and po
%\noindent \textbf{Inner and outer Kernels Dimensions.} \textcolor{blue}{$\log(|\mathcal{A}|)$ TBD}\\
The following proposition proved in Appendix \ref{app:prop1} gives bounds on $m$ and $m'$:
\begin{proposition} [HKSE approximation] Let $\mathcal{A}$ be the vocabulary embedded in a vector space. Let $|\mathcal{A}|$ be the size of the vocabulary. Let $s$ be the maximum sentence length. Let $\hat{\mu}_{n_1}(\rho_1)= \frac{1}{n_1}\sum_{i=1}^{n_1}\Phi_{\gamma,d}(a_i)$, and $\hat{\mu}_{n_2}(\rho_2)= \frac{1}{n_2}\sum_{i=1}^{n_2}\Phi_{\gamma,d}(b_i)$, on two different sentences $\rho_1,\rho_2$, with words $\{a_1,\dots a_{n_1}\}$ and $\{b_1,\dots b_{n_2}\}$ respectively. Let $\hat{K}(\rho_1,\rho_2)=\scalT{\Phi_{\eta,m}\left(\hat{\mu}_{n_1}(\rho_1)\right)}{\Phi_{\eta,m}\left(\hat{\mu}_{n_2}(\rho_2)\right)}$ Let $\epsilon,\delta >0$, for $m\geq  \frac{1}{2\delta^2}\log(|\mathcal{A}|^2/\epsilon)$ and $m'\geq  \frac{1}{2\delta^2}\log(|\mathcal{A}|^{2s}/\epsilon)$ we have with probability $1-2\epsilon$:\\
$\frac{1}{c(\eta,\delta)} K(\rho_1,\rho_2) - \delta  \leq \hat{K}(\rho_1,\rho_2) \leq  c(\eta,\delta) K(\rho_1,\rho_2)  +\delta ,$
where $c= \exp(\frac{3\eta \delta}{2})$.
\end{proposition}
Informally for the 2 layers kernel the estimation error is multiplicative and additive. The inner word level dimension of HKSE scales logarithmically with the vocabulary size $m =O(\log(|\mathcal{A}|))$, and the outer dimension scales linearly with the maximum sentence length $s$ an logarithmically with the vocabulary size  $m'=O(s\log(|\mathcal{A}|) )$.
\begin{remark}
We give in Appendix \ref{app:Gauss2Vec} a probabilistic interpretation of HKSE as a Gaussian Embedding of sentences.
\end{remark}

\section{Relation to Previous work}
%Joint modeling of language and vision is a rapidly growing field
We focus in this section on some recent works that use bidirectional maps in the retrieval tasks and their relation to this paper.
\cite{Klein}, and \cite{Gong} used CCA to build a joint representation using the cosine similarity. In both works a symmetric weighting of the CCA canonical weights was used, i.e for an image caption pair $(x^*,y^*)$, a joint embedding of the form $(\Sigma^{\alpha}U^{\top}x^*,\Sigma^{\alpha}V^{\top}y^*)$, was used, where $\alpha=0$, in the case of  \cite{Klein}  and $\alpha>0$ in the case of \cite{Gong}. 
The symmetric weighting in \cite{Gong} was  found to improve performance and is not theoretically motivated as in the asymmetric weighting of this paper. Indeed our experimental results show that asymmetric weighting gives higher performance. Skip thought vectors introduced in \cite{skipthoughts} for representing sentences were used for bidirectional search in conjunction with VGG features \cite{vgg} on the image side, with linear embeddings learned with a discriminative triplets loss, instead of a CCA loss. Order embeddings introduced in \cite{vendrov2015order} uses an order similarity  rather than the cosine similarity  and learns an end to end sentence embedding with an LSTM. Another end to end approach using deep neural networks on image an text features under a structured loss was introduced  \cite{WangLL15}.
% and non linear features in CITE WANG . %AW-G-Tikh-CCA and non linear  HKSE lead to better generalization in retrieval tasks.  %\textcolor{red}{Socher}
\section{Numerical Results}\label{sec:exp}
%\subsection{Retrieval Results}
\noindent \textbf{Datasets.} We performed image annotation and search tasks on the MSCOCO benchmark \cite{MSCOCO}, and  Flickr 8K and 30 K benchmarks \cite{Flickr8k,Flickr30k} using our task dependent  asymmetrically weighted CCA, as described in Table \ref{tab:asymW}. Retrieval was performed using cosine similarities given in Equations \eqref{eq:imSearchcca1} and \eqref{eq:imAnnotationcca1}. For details on data splits and experimental protocol check Appendix \ref{app:splits} . We report for both
tasks the recall rate at one result, five results, or ten first
results (r@1,5,10), as well as the median rank of the first ground truth retrieval on the test set.\\

\noindent \textbf{Image/Text Features.} For feature extraction we follow \cite{Klein}, and use on the image side the VGG CNN representation \cite{vgg}, where each image was rescaled  to have smallest side 384 pixels, and then cropped in $10$  ways into 224 by 224 pixel images: the four corners, the center, and their x-axis mirror image. The mean intensity of each crop is subtracted in each color channel, and then encoded by VGG19 (the final FC -4096 layer). The average of the resulting 10 feature vectors corresponding to each crop is used as the  image representation. We also use the residual convolutional neural network Resnet101 ($101$ layers)\cite{he15deepresidual}. We don't rescale or crop the image , we encode the full size image  with the Resnet and  apply spatial average pooling to the last layer  resulting in a vector of dimension $m_x= 2048$. 
% (Res B)  where we apply  spatial adapative average pooling to the last $\textcolor{blue}{x}$ layers of Resnet, that we concatenate, this results in a vector of dimension $\textcolor{blue}{m_x= }$.  Stacking resnet layers gives us a form of multi-scale description of the image in high dimensions. CCA captures more correlations in higher dimensions.
 On the text side, we use two sentence embeddings 1) $\rm{HKSE}(k^w,k^s)$, introduced in this paper, using word2vec  available on \url{code.google.com/p/word2vec/} (throughout our experiments, $\gamma$ was set to be the inverse of the squared median pairwise distances of words in the vocabulary, $\eta$ was fixed to $0.01$, for MSCOCO we use $m=2000$ and $m'=3000$, for Flickr we use $m=2000$ and $m'=2000$). Note that HKSE(lin,lin), corresponds to the mean word2vec baseline as in \cite{Klein}. 2) Skip thought vectors introduced in \cite{skipthoughts}, which encodes sentences to vectors using an LSTM. Image and text features were centered before learning CCA. For AW-CCA we used cosine similarity that we find to outperform $\ell_2$, and performed search and annotation as given in Equations \eqref{eq:imSearchcca1}, and \eqref{eq:imAnnotationcca1} respectively.

\noindent \textbf{Empirical Validation of Our Approach.}\label{app:Opt}\\
We confirm in this section empirically our main claims and contributions:\\
%\begin{itemize}[noitemsep]
1) \emph{Query Generation Versus Search Generation:} In order to confirm our claims in Section \ref{sec:cca}, on the optimality of the asymmetric weighting in the query generation approach, we perform bidirectionnal retrieval  on MS-COCO using  image embedding of the form  $\Sigma^{\alpha} U^{\top}$ and caption embedding of the form $\Sigma^{1-\alpha}V^{\top},$ where $\alpha \in [0,1]$. We plot in Figure \ref{fig:opasym} r@10 versus $\alpha$ on the validation set, for image annotation and image search.  The cosine similarity was used between embeddings. In this experiment we used VGG image features  and average word2vec sentence embedding. For each $\alpha$ we perform a thorough cross validation using truncated SVD CCA, and report the best r@10 for image annotation and image search. We see that $\alpha= 0$, i.e the embedding $(U^{\top},\Sigma V^{\top})$ is optimal for image annotation as predicted by the query generation approach for image annotation. Similarly we see that $\alpha=1$, i.e the embedding $(\Sigma U^{\top},V^{\top})$, is optimal for image search. This confirms our claim that the asymmetric weighting should be on the search space.  To illustrate the organization of the space with asymmetric CCA, we refer the reader to t-SNE plots in  Appendix \ref{sec:tsne}.\\ 
\begin{figure}[ht!]
  \centering
  \includegraphics[scale=0.23]{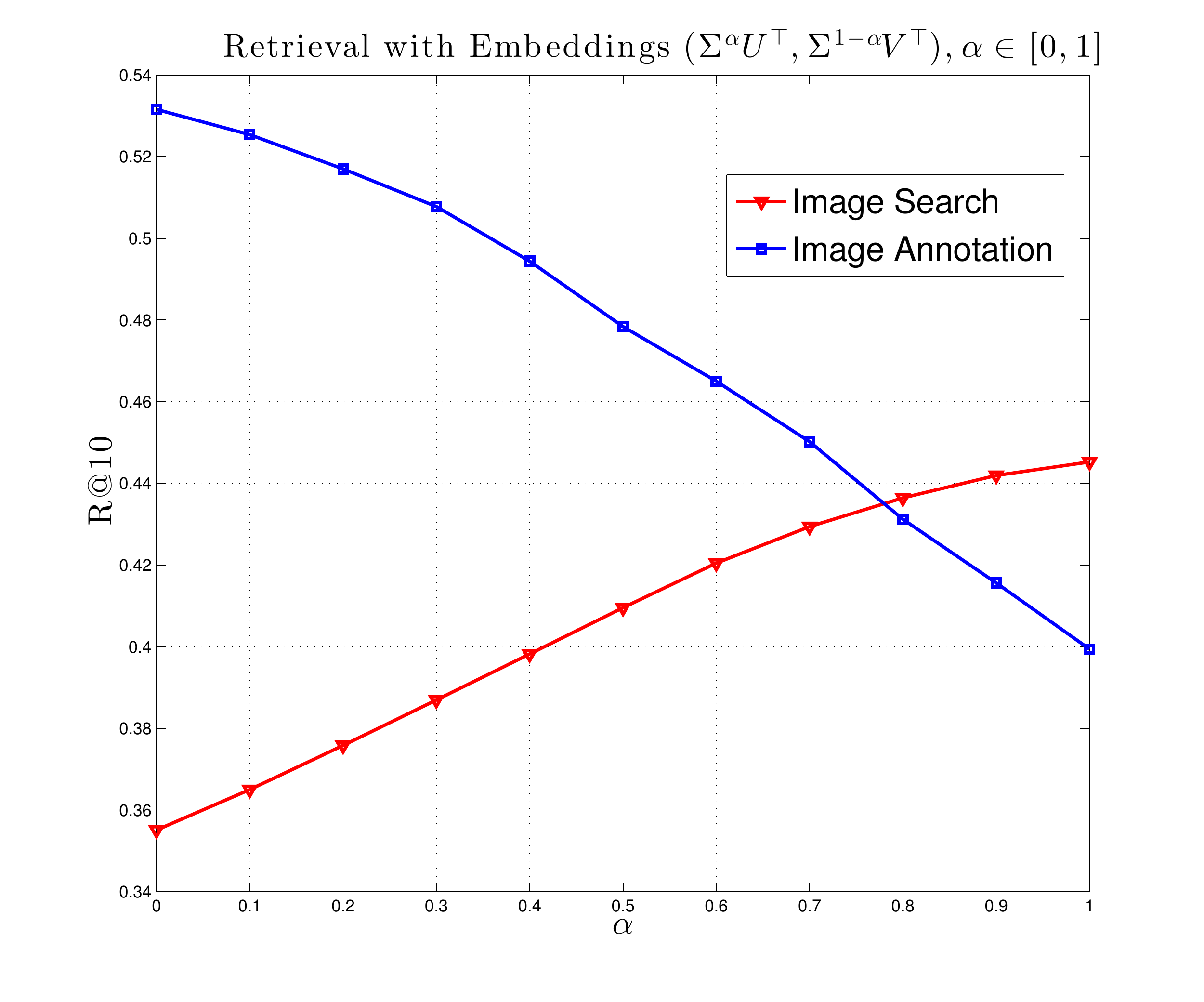}
  \caption{Optimality of the asymmetric weighting in the query generation approach.}
  \label{fig:opasym}
  \vskip -0.2 in
\end{figure}
 
\noindent 2) \emph{Asymmetric Weighting versus CCA and Symmetric Weighting:} On MS-COCO, using same VGG/average word2vec features, we perform bidirectional retrieval using embeddings of the form $(\Sigma^{\alpha}U^{\top},\Sigma^{\alpha}V^{\top})$. Note that $\alpha =0$ corresponds to CCA as used in \cite{Klein}, and $\alpha>0$ corresponds to the symmetric weighting introduced in \cite{Gong}, we see in Table \ref{tab:SymWeight} that AW-CCA outperforms both approaches. \\
   \begin{table*}[ht]
\hspace*{-0.005in}\centering
\begin{tabular}{|l|lll|lll|}
\hline
& \multicolumn{3}{|c|}{Image search} & \multicolumn{3}{|c|}{Image annotation} \\
& r@1 & r@5 & r@10 & r@1 & r@5 & r@10  \\
%& & & & rank & rank & & & & rank & rank& rank\\
\hline

Symmetric Weighting                                                                &      &      &      &      &      &          \\
$(\Sigma^{\alpha} U^{\top},\Sigma^{\alpha} V^{\top})$, $\alpha =0$ (CCA)   &9.36 & 26.13 & 37.30& {13.96} & {34.30}& {47.10}  \\
$(\Sigma^{\alpha} U^{\top},\Sigma^{\alpha} V^{\top})$, $\alpha =1$   &{11.21}& {28.95}& {40.24}&{13.72} &{34.00}&  {45.86} \\
$(\Sigma^{\alpha} U^{\top},\Sigma^{\alpha} V^{\top})$, $\alpha =2$ &10.04&  26.54&  37.46& 11.16&  29.02&  39.56\\
$(\Sigma^{\alpha} U^{\top},\Sigma^{\alpha} V^{\top})$, $\alpha =3$ &8.30& 23.53&33.78&9.36&24.74&35.56\\
$(\Sigma^{\alpha} U^{\top},\Sigma^{\alpha} V^{\top})$, $\alpha =4$ &7.09& 20.79& 30.81& 7.88& 21.58&  32.22\\
$(\Sigma^{\alpha} U^{\top},\Sigma^{\alpha} V^{\top})$, $\alpha =5$ &6.16&18.63& 27.79& 6.76&19.76& 29.46\\
$(\Sigma^{\alpha} U^{\top},\Sigma^{\alpha} V^{\top})$, $\alpha =6$ &5.55&16.76&25.31& 6.08&17.80&27.12\\
\hline
\hline
Task Dependent Asymmetric Weighting:                                                            &      &      &      &      &      &          \\                                                            
Image Search:$(\Sigma U^{\top},V^{\top})$, Image Annotation:$(U^{\top},\Sigma V^{\top})$   & {12.92} & {32.38} & {44.52} & {17.84} & {40.42} & {53.16} \\
 \hline
\end{tabular}
\vspace{.2cm}
\caption{The asymmetric weighting we propose outperforms the symmetric weighting \cite{Gong} and CCA as shown in this table, and boosts performance (retrieval rates in \%) of both tasks. }
\label{tab:SymWeight}
\end{table*} 

\noindent 3) \emph{Guided Tikhonov Cross Validation Speed/Performance Tradeoff:} 
We report in Table \ref{tab:asymWTIMING} the timing (in hours) of computing regularization path of AW- CCA, with T-SVD and Tikohnov, on MSCOCO. The parameter grid for both cases is 20$\times$20. Experiments were conducted on a single Intel Xeon CPU E5-2667, 3.30GHz, with 265 GB of RAM and 25.6 MB of cache. We see that T-SVD is faster to cross validate than Tikhonov : $2$x speedup for VGG/w2v ($m_x$=4096, $m_y$=300), $6$x speedup for VGG/skip thoughts ($m_x$=4096, $m_y$=4800). Thanks to the speed advantage of 
T-SVD we confirm that T-SVD Guided Thikhonov regularization introduced in Section \ref{sec:BGTikhonov.} incurs a small loss in accuracy  with respect to full Tikhonov validation on Flickr 30K dataset.
 \begin{table}[H]
%In term of performance we see in Tables \ref{tab:MSCOCO} and \ref{tab:MSCOCO2} that (Guided by T-SVD) Tikhonov regularization has a slightly better performance than truncated SVD. }
\label{sample-table}
\begin{center}
\begin{tabular}{lllll}
\multicolumn{1}{c}{\bf AW- CCA} &\multicolumn{1}{c}{\bf CPU time}
\\ \hline \\
T-SVD (VGG/w2v)  & $1.04 $ h\\
Tikhonov- CCA (VGG/w2v)  &$1.94$ h\\     
T-SVD (VGG/Skip)& $6.92 $ h\\
 Tikhonov  (VGG/Skip)  &$45.31$ h     
\end{tabular}
\end{center}
\caption{ Timing in hours of CCA Cross Validation: T-SVD versus Tikhhonov.} 
\label{tab:asymWTIMING}
\end{table}
\vskip -0.2 in
\begin{table}[ht!]
\hspace*{-0.05 in}
\centering
\small\addtolength{\tabcolsep}{-3 pt}
\begin{tabular}{{|l|lll|lll|l}}
\hline
& \multicolumn{3}{|c|}{Image search} & \multicolumn{3}{|c|}{Image annotation} \\
 AW-CCA& r@1 & r@5 & r@10  &  r@1 & r@5 & r@10   \\
& & &   & & & \\
\hline
Tikh &\textbf{23.08}&\textbf{50.62}&\textbf{63.53}&\textbf{30.66}&\textbf{59.1}&\textbf{70.13} \\                                                                  
T-SVD  & 20.61&46.91&59.9&27.0&55.5&67.80  \\  
 G-Tikh &\textbf{22.40}&\textbf{49.94}&\textbf{62.62}&\textbf{30.56}&\textbf{58.43}&\textbf{69.96}\\  
 \hline                                                                
\end{tabular}
\vspace{.2cm}
\caption{Mean results (in \%) of the test splits on the Flickr 30 K, in average VGG/ Average w2vec setup. We see that G-Tikh and Tikh are on par.}% \textcolor{blue}{med r cv, we HAVE R@1 for full tikh not the rest}}
\label{tab:flickrComp}
\vspace{-0.2 in}
\end{table} 

\noindent \textbf{Discussion of the Results.} We report here our results using  the Asymmetrically Weighted TSVD-Guided-Tikhonov  CCA (AW-G-Tikh-CCA) given in Algorithm \ref{ALG:GTikhonovcca} (Appendix \ref{ap:algorithms}) where we select the model corresponding to the maximum r@1 on the validation set (regularization paths are given in Appendix \ref{app:RegPaths}). When we use  VGG for image encoding we observe in Tables \ref{tab:MSCOCO} and \ref{tab:flickr} that 1)   AW-G-Tikh-CCA with the same Mean Vec features (HKSE(lin,lin)) outperforms consistently the CCA baseline in \cite{Klein}, this is  mainly due to the asymmetric weighting and the query generation approach, as well as the efficient model selection introduced in this paper. 2) HKSE(rbf,rbf) consistently outperforms the mean vector baseline as well as skip thoughts vectors and can be used therefore as a simple yet strong off the shelf sentence embedding. HKSE(rbf,rbf) outperforms the fisher vector representation  \cite{Klein}, on the MSCOCO Benchmark. Note that, Fisher vectors are learned on the dataset, and HKSE is fully unsupervised. For the same image features (VGG) end to end  order embeddings (VGG)\cite{vendrov2015order} outperforms HKSE on image search on COCO. %Order embeddings use an order similarity and the sentence embedding is learned. \\
We see from Tables \ref{tab:MSCOCO} and \ref{tab:flickr} that  encoding images with Resnet and spatial average pooling, in conjunction with the concatenation of [HSKE(lin,rbf)||HKSE(rbf,rbf)]  boosts the performance of AW-CCA and achieve state of the art performance on both tasks.  Note that we get significant boosts on both image search and annotation tasks since the asymmetric weighting of CCA gives us for free an optimized system for each task, while competing methods optimize jointly for both tasks using the same embeddings.   
  %\textcolor{blue}{more comments on results ...}%
%\noindent \textbf{Retrieving A Caption Set.} We consider the problem of assigning to each image the set of five ground truth captions. In order to return a single caption  we  select within  the returned set the caption with largest cosine. Note that for MSCOCO and Flickr datasets $S= \{(x_i, \{y_{i,j },j=1\dots 5\}),i=1\dots N \}$ the  correlation matrix :
%$C_{XY}= \sum_{i=1}^N x_i \sum_{j=1}^5 y^{\top}_{i,j}.$
%Hence the CCA objective is naturally correlating the image $x_i$ with the unnormalized average caption $\sum_{j=1}^5 y^{\top}_{i,j}$.
%Hence the natural representation of the set of five captions by their average.  As more words are available HKSE(rbf,rbf) gets a better estimate of the underlying distribution and achieves state of the art performance on annotation (Ann Pack in Tables \ref{tab:MSCOCO} and \ref{tab:flickr} ). 
%  

\begin{table*}[ht]
\vspace{-0.8cm}
\hspace*{-0.0005in}
\centering
\small\addtolength{\tabcolsep}{-3.5pt}
\begin{tabular}{{|l|llll|llll|l}}
\hline
& \multicolumn{4}{|c|}{Image search} & \multicolumn{4}{|c|}{Image annotation} \\
& r@1 & r@5 & r@10 & med r &  r@1 & r@5 & r@10 & med r  \\
%& & & & rank  & & & & rank \\
\hline
1K test images:                                                                    &      &      &      &      &       &      &      &                \\
~CCA+ Mean Vec {\cite{Klein} }                                                                 & 24.2 & 56.4 & 72.4 & 4.0 & 33.2 & 61.8 & 75.1 & 3.0 \\
%~Mean Vec (vocabPCA+AW-T-SVD-CCA)                                                                    & {26.67} & {59.84} & {74.90} & {4.0}  & {36.28} & {68.16}& {80.72} & {2.0}  \\
~AW-G-Tikh-CCA  +VGG+ HKSE(lin,lin) (Mean Vec)                                                                     & {28.14} & {61.33} & {76.64} & {3.0}  & {37.16} & {68.78}& {81.36} & {2.0}  \\
%~Mean Vec (vocabPCA+AW-\emph{kernel}-T-SVD-CCA)                                                                    &{29.80}&{64.09}&{78.43}&{3.0}     & {39.92} & {71.34}& {83.52} & {2.0}  \\
~AW-G-Tikh-CCA+VGG+HKSE (lin,rbf)                                                                        & {30.99} & {65.43} & {79.49} & {3.0}  & {39.94} & {71.82}& {84.1} & {2.0}  \\
~AW-G-Tikh-CCA+VGG+HKSE (rbf,lin)                                                          & {29.54} & {63.79} & {78.70} &{3.0} & {41.10} & {70.90} & {82.46} & {2.0}   \\
~AW-G-Tikh-CCA+VGG+HKSE (rbf,rbf)                                                                & {32.70} & {67.51} & {81.14} & {3.0} & {43.64} & {74.94} & {85.68} & {2.0}   \\
~AW-G-Tikh-CCA+VGG+[HKSE (lin,rbf)||HKSE(rbf,rbf)]               & {34.84} & {69.92} & {82.73} & {3.0}  & {46.54} & {77.50}& {87.22} & {2.0}  \\
%~AW-G-Tikh-CCA +VGG+HKSE(rbf,rbf) (Ann Pack)       & NA&NA &NA&NA   & \textbf{55.12} & \textbf{86.36}& \textbf{94.24} & \textbf{1.0}  \\
~AW-G-Tikh-CCA+ResNet+HKSE (lin,lin)                                         & {32.95} & {67.12} & {81.07} & {3.0}  & {43.1} & {75.92}& {87.24} & {2.0}  \\
~AW-G-Tikh-CCA+ResNet+[HKSE (lin,rbf)||HKSE(rbf,rbf)]               & \textbf{39.68} & \textbf{73.98} & \textbf{86.23} & \textbf{2.0}  & \textbf{55.16} & \textbf{83.84}& \textbf{92.52} & \textbf{1.0}  \\
%~AW-G-Tikh-CCA+Res(B)+HKSE (lin,lin)                                          & {34.66} & {68.76} & {82.21} & {3.0}  & {45.82} & {78.18}& {89.02} & {2.0}  \\
%~AW-G-Tikh-CCA+Res(B)+[HKSE (lin,rbf)||HKSE(rbf,rbf)]                                                                       & {40.67} & {75.24} & {86.98} & {2.0}  & {55.24} & {84.86}& {93.34} & {1.0}  \\    
%~Skip thoughts (AW-T-SVD-CCA)                                                                      & {27.76} & {62.25} & {76.30} & {3.0}  & {38.42} & {70.10}& {82.84} & {2.0}   \\
~AW-G-Tikhonov-CCA+VGG+Skip thoughts                                                                    & {29.14} & {63.74} & {77.53} & {3.0}  & {39.24} & {70.44}& {82.68} & {2.0}   \\
~Triplets loss +VGG+Skip thoughts {\cite{skipthoughts} }                                                                       & 25.9 & 60 & 74.6 &NA & {33.8} & {67.7}& {82.1} & {NA} \\
~BRNN+VGG \cite{Karpathy}~                                                                            & 20.9 & 52.8 & 69.2 & 4.0  & 29.4 & 62.0 & 75.9 & 2.5     \\
~CCA +VGG+Fisher GMM+HGLMM   {\cite{Klein} }                                                                               & 25.6 & 60.4 & 76.8 & 4.0  & 38.9 & 68.4 & 80.1 & 2.0  \\
~Order Embedding+VGG \cite{vendrov2015order}                                                                  & {37.9} & {NA}& {85.9} & {2.0} & {46.7} & {NA} & { 88.9} & {2.0}    \\

\hline
5K test images:                                                                    &      &      &      &      &      &      &      &            \\
~CCA +VGG+Mean Vec   {\cite{Klein} }                                                                    & 10.3 & 27.2 & 38.4 & 18.0 & 12.8 & 32.1 & 44.6 & 14.0 \\
%~Mean Vec  (vocabPCA+ AW-T-SVD-CCA)                                                                   & {11.86} & {30.56} & {42.30} & {15.0} & {17.22} & {39.60} & {52.84} & {9.0}   \\
~AW-G-Tikh-CCA +VGG+HKSE(lin,lin) (Mean Vec)                                                                 & {12.91} & {32.17} &{44.24} & {14.0} & {17.9} &{40.30} & {53.12} & {9.0}   \\
%~Mean Vec  (vocabPCA+ AW-\emph{kernel}-T-SVD-CCA)                                                                   & {12.93} & {34.20} & {47.04} & {12.0} & {18.74} & {42.62} & {56.44} & {8.0}   \\
~AW-G-Tikh-CCA+VGG +HKSE (lin,rbf)                                                                 & {14.24} & {35.39} & {48.34} & {11.0} & {19.06} & {43.02} & {57.02} & {8.0}   \\
~AW-G-Tikh-CCA+VGG+HKSE (rbf,lin)                                                             & {13.46} & {33.74} & {46.08} &{13.0} & {19.98} & {44.50} & {56.90} & {7.0}   \\
~AW-G-Tikh-CCA+VGG+HKSE (rbf,rbf)                                                             & {15.44} & {37.44} & {50.69} & {10.0} & {22.14} & {47.76} &{60.68} &{6.0}   \\
~AW-G-Tikh-CCA+VGG+[HKSE (lin,rbf)||HKSE(rbf,rbf)]               & {16.63} & {40.35} & {53.60} & {9.0}  & {23.70} & {50.32}& {63.14} & {5.0}  \\
%~AW-G-Tikh-CCA  + HKSE(rbf,rbf) (Ann Pack)& NA&NA &NA&NA & \textbf{32.46} & \textbf{62.20}& \textbf{75.00} & \textbf{3.0}  \\
~AW-G-Tikh-CCA+ResNet+HKSE (lin,lin)                                                                        & {15.69} & {37.64} & {50.29} & {10.0}  & {22.58} & {47.3}& {60.98} & {6.0}  \\
~AW-G-Tikh-CCA+ResNet+[HKSE (lin,rbf)||HKSE(rbf,rbf)]            & \textbf{20.22} & \textbf{45.05} & \textbf{58.53} & \textbf{7.0}  & \textbf{31.48} & \textbf{60.54}& \textbf{72.32} & \textbf{3.0}  \\
%~AW-G-Tikh-CCA+Res(B)+HKSE (lin,lin)                                       & {16.69} & {39.51} & {52.22} & {9.0}  & {24.48} & {50.72}& {64.34} & {5.0}  \\
%~AW-G-Tikh-CCA+Res(B)+[HKSE (lin,rbf)||HKSE(rbf,rbf)]                & {20.66} & {46.39} & {59.84} & {7.0}  & {31.94} & {60.54}& {73.20} & {3.0}  \\
%~Skip thoughts (AW-T-SVD-CCA)                                                                         & {11.81} &{31.96} & {45.15} & {13.0} & {17.68} & {41.44} & {54.68} & {8.0}   \\
~AW-G-Tikhonov-CCA+VGG+Skip thoughts                                                                & {12.83} &{33.73} & {47.04} & {12.0} & {18.5} & {42.24} & {55.34} & {8.0}   \\
~Triplets loss +VGG+Skip thoughts {\cite{skipthoughts} }                                                                       &  NA& NA &  NA &  NA &  NA &  NA &  NA & NA \\
~BRNN+VGG      \cite{Karpathy}                                                                            & 8.9  & 24.9 & 36.3 & 19.5   & 11.8 & 32.5 & 45.4 & 12.2   \\
%~BRNN  \cite{Karpathy}                                                                            & 8.9  & 24.9 & 36.3 & 19.5   & 11.8 & 32.5 & 45.4 & 12.2   \\
%~BRNN  \cite{Karpathy}                                                                           &11.8& 32.1& 44.7& 12.4&  16.5& 40.6& 54.2& 7.6 \\
%~BRNN  \cite{Karpathy}                                                                           &11.8& 32.1& 44.7& 12.4&  16.5& 40.6& 54.2& 7.6 \\
%~GMM                                                                 & 10.5 & 28.0 & 39.4 & 17.0 & 62.3 & 17.0 & 38.1 & 49.8 & 11.0 & 52.2 & 58.0 \\
%~LMM                                                                 & 10.5 & 28.0 & 39.7 & 17.0 & 61.3 & 16.4 & 38.1 & 50.4 & 10.0 & 51.2 & 57.0 \\
%~HGLMM                                                               & 11.1 & 28.7 & 40.2 & 16.0 & 59.0 & 16.7 & 38.4 & 51.0 & 10.0 & 47.8 & 56.2 \\
~CCA +VGG+Fisher GMM+HGLMM {\cite{Klein} } & 11.2 & 29.2 & 41.0 & 16.0 & 17.7 & 40.1 & 51.9 & 10.0 \\
~Order Embedding+VGG \cite{vendrov2015order}                                                                   & {18.0 } & {NA} & {57.6} & {7.0}  & {23.3} & {NA}& {65} & {5.0}  \\ 

\hline
\end{tabular}
\caption{Mean results of the test splits on the MSCOCO benchmark (in \%). }
\label{tab:MSCOCO}
\vspace{-0.4 cm}
\end{table*}

\begin{table*}[ht]
%\vspace{-0.85cm}
\hspace*{-0.005in}
\centering
\small\addtolength{\tabcolsep}{-3.5pt}
\scalebox{1}{
\begin{tabular}{{|l|llll|llll|l}}
\hline
& \multicolumn{4}{|c|}{Image search} & \multicolumn{4}{|c|}{Image annotation} \\
& r@1 & r@5 & r@10 & med r&  r@1 & r@5 & r@10 & med r  \\
%& & & & rank  & & & & rank \\
\hline
Flickr 30 K                                                             &      &      &      &      &       &      &      &                \\
~CCA + VGG+Mean Vec {\cite{Klein} }                                                          & 20.5 & 46.3& 59.3& 6.8&24.8& 52.5& 64.3& 5.0\\                        
~AW-G-Tikh-CCA+VGG+ HKSE (lin,lin) (Mean Vec) &{22.40}&{49.94}&{62.62}&{6.0}&{30.56}&{58.43}&{69.96}&{4.0} \\                                                                  
~AW-G-Tikh-CCA+VGG+HKSE (rbf,rbf)        &{25.80}&{54.72}&{66.52}&{4.0}&{32.5}&{61.03}&{73.23}&{3.0}\\   
~AW-G-Tikh-CCA+VGG+[HKSE (lin,rbf)||HKSE(rbf,rbf)]               & {26.43} & {55.03} & {66.39} & {4.0}  & {33.73} & {63.33}& {75.17} & {3.0}  \\ 
%~AW-G-Tikh-CCA +VGG+HKSE(rbf,rbf)  (Ann Pack)                                                            & NA&NA &NA&NA & \textbf{43.1} & \textbf{74.4}& \textbf{84.3} & \textbf{2.0}  \\
~AW-G-Tikh-CCA+ResNet+HKSE (lin,lin)                                                                        & {28.29} & {56.81} & {69.0} & {4.0}  & {37.0} & {66.67}& {78.23} & {2.0}  \\
~AW-G-Tikh-CCA+ResNet+[HKSE (lin,rbf)||HKSE(rbf,rbf)]                                              & \textbf{31.48} & \textbf{60.71} & \textbf{72.09} & \textbf{3.0}  & \textbf{43.83} & \textbf{71.87}& \textbf{81.87} & \textbf{2.0}  \\
%~AW-G-Tikh-CCA+Res(B)+HKSE (lin,lin)                                                                        & {} & {} & {} & {}  & {} & {}& {} & {}  \\
%~AW-G-Tikh-CCA+Res(B)+[HKSE (lin,rbf)||HKSE(rbf,rbf)]                                                                       & {} & {} & {} & {}  & {} & {}& {} & {}  \\ 
%~Mean Vec (vocabPCA+AW-T-SVD-CCA)   & 20.53&46.83&59.84&6.0&27.1&55.87&67.80&4.0  \\                                                                
%~Mean Vec (vocabPCA+AW-\emph{kernel}-T-SVD-CCA)                                                                    &{29.80}&{64.09}&{78.43}&{3.0}     & {39.92} & {71.34}& {83.52} & {2.0}  \\
%~Mean Vec (vocabPCA+AW-\emph{kernel}-G-Tikhonov-CCA)                                                                    & \textbf{30.80} & \textbf{65.16} & \textbf{79.45} & \textbf{3.0}  & \textbf{40.22} & \textbf{71.92}& \textbf{83.40} & \textbf{2.0}  \\
%~Skip thoughts (AW-T-SVD-CCA) &19.89&45.91&58.18&7.0&28.23&54.97&66.80&4.0\\                                                                     
~AW-G-Tikhonov-CCA +VGG+ Skip thoughts&22.18&48.9&60.92&6.0&28.23&56.50&68.36&4.0\\       
~BRNN+VGG \cite{Karpathy}~                                                                            & 15.2 & 37.7& 50.5 & 9.2  & 22.2 & 48.2 & 61.4 & 4.8     \\
                                             %~Skip thoughts +Triplets loss {\cite{skipthoughts} }                                                                       & 25.9 & 60 & 74.6 &NA & {33.8} & {67.7}& {82.1} & {NA} \\
%~GMM                                                                 & 24.7 & 58.7 & 75.6 & 4.0  & 13.3 & 38.3 & 66.1 & 79.2 & 3.0  & 12.1 & 13.9 \\
%~LMM                                                                 & 24.8 & 59.2 & 75.4 & 4.0  & 12.8 & 39.1 & 68.4 & 79.4 & 2.0  & 12.2 & 13.9 \\
%~HGLMM                                                               & 25.1 & 59.7 & 76.5 & 4.0  & 12.7 & 38.7 & 68.4 & 81.0 & 2.0  & 11.3 & 13.7 \\
~CCA+VGG+ Fisher GMM+HGLMM{\cite{Klein} }                                                                               & 25.0 & 52.7 & 66.0 & 5.0  & {35.0} &{62.0} & {73.80} & {3.0}  \\
\hline
Flickr 8K                                                                 &      &      &      &      &      &      &      &            \\
%~Mean Vec  (vocabPCA+ AW-T-SVD-CCA)  &16.52&40.34&54.02&9.0&23.30&50.40&64.10&5.0\\  
~CCA + VGG+Mean Vec    {\cite{Klein} }                                                                  &{19.1}& {45.3}& {60.4}& {7.0}& {22.6} &{48.8}& {61.2}& {6.0}\\                                                               
~AW-G-Tikh-CCA+VGG+HKSE (lin,lin) (Mean Vec)      &18.62&44.82&58.88&7.0&{23.10}&{50.8}&{63.00}&{5.0}\\                                                                 %~Mean Vec  (vocabPCA+ AW-\emph{kernel}-T-SVD-CCA)                                                                   & {12.93} & {34.20} & {47.04} & {12.0} & {18.74} & {42.62} & {56.44} & {8.0}   \\
%~Mean Vec  (vocabPCA+ AW-\emph{kernel}-G-Tikhonov-CCA)                                                                   & \textbf{13.66} & \textbf{35.27} & \textbf{47.92} & \textbf{12.0} & \textbf{19.16} & \textbf{43.22} & \textbf{57.14} & \textbf{8.0}   \\
%~Skip thoughts (AW-T-SVD-CCA)                                                                         &13.78&36.30&49.04&11.0&20.90&45.80&58.90&7.0   \\
~AW-G-Tikh-CCA+ VGG+HKSE (rbf,rbf)         &{19.6}&{45.66}&{58.00}&{7.0}&{25.5}&{53.4}&{67.60}&{5.0}\\   
~AW-G-Tikh-CCA+VGG+[HKSE (lin,rbf)||HKSE(rbf,rbf)]               & {19.36} & {45.34} & {57.08} & {7.0}  & {26.0} & {52.3}& {67.1} & {5.0}  \\  
~AW-G-Tikh-CCA+ResNet+HKSE (lin,lin)                       & {23.08} & {52.18} & {65.5} & {5.0}  & {32.8} & {63.4}& {74.8} & {3.0}  \\
~AW-G-Tikh-CCA+ResNet+[HKSE (lin,rbf)||HKSE(rbf,rbf)]    & \textbf{23.74} & \textbf{52.48} & \textbf{65.72} & \textbf{5.0}  & \textbf{34.6} & \textbf{65.7}& \textbf{75.7} & \textbf{3.0}  \\
%~AW-G-Tikh-CCA+Res(B)+HKSE (lin,lin)                                                                        & {} & {} & {} & {}  & {} & {}& {} & {}  \\
%~AW-G-Tikh-CCA+Res(B)+[HKSE (lin,rbf)||HKSE(rbf,rbf)]                                                                       & {} & {} & {} & {}  & {} & {}& {} & {}  \\ 
%~AW-G-Tikh-CCA  + HKSE(rbf,rbf)  (Ann Pack)       & NA&NA &NA&NA&\textbf{36.1} & \textbf{68.7}& \textbf{79.4} & \textbf{2.0}  \\
~AW-G-Tikhonov-CCA+VGG+Skip thoughts           &17.52&43.76&57.92&7.0&21.70&50.20&63.70&5.0\\                      
~BRNN+VGG  \cite{Karpathy}                                                                           &11.8& 32.1& 44.7& 12.4&  16.5& 40.6& 54.2& 7.6 \\
% ~BRNN  \cite{Karpathy}                                                                           &11.8& 32.1& 44.7& 12.4&  16.5& 40.6& 54.2& 7.6 \\
%~BRNN  \cite{Karpathy}                                                                           &11.8& 32.1& 44.7& 12.4&  16.5& 40.6& 54.2& 7.6 \\
                                     
%~Skip thoughts +Triplets loss {\cite{skipthoughts} }                                                                           &  NA& NA &  NA &  NA &  NA &  NA &  NA & NA \\
%~GMM                                                                 & 10.5 & 28.0 & 39.4 & 17.0 & 62.3 & 17.0 & 38.1 & 49.8 & 11.0 & 52.2 & 58.0 \\
%~LMM                                                                 & 10.5 & 28.0 & 39.7 & 17.0 & 61.3 & 16.4 & 38.1 & 50.4 & 10.0 & 51.2 & 57.0 \\
%~HGLMM                                                               & 11.1 & 28.7 & 40.2 & 16.0 & 59.0 & 16.7 & 38.4 & 51.0 & 10.0 & 47.8 & 56.2 \\
~CCA+VGG+Fisher GMM+HGLMM  {\cite{Klein} }&{21.2}& {50.0}& {64.8}& {5.0}&{31.0}&{59.3}& {73.7}& {4.0}\\
\hline
\end{tabular}}
\caption{Mean results of the test splits on the Flickr 30 K and 8K benchmarks (in \%).}
\label{tab:flickr}
\vspace{-0.2in}
\end{table*}

\section{Conclusion}
%\vspace{-0.1 in}
In this paper we  showed that  the query generation approach in information retrieval \cite{Relevance},  can be applied to bidirectional retrieval, where we map the search space to the query space via an asymmetric weighting of regularized CCA. Asymmetric weighting improves the performance of the bidirectional retrieval tasks.
We also presented a computationally efficient cross validation for regularized CCA, that allows for a better model selection and hence contributes also in improving the retrieval performance. 
Finally we presented the Hierarchical Kernel Sentence Embedding that is of independent interest , and that generalizes the mean word2vec as a mean for aggregation of word embeddings and outperforms off the shelf sentence embeddings in bidirectional retrieval.

 \bibliographystyle{unsrt}
 \bibliography{simplex}

\begin{thebibliography}{10}

\bibitem{Karpathy}
Andrej Karpathy and Fei-Fei Li.
\newblock Deep visual-semantic alignments for generating image descriptions.
\newblock In {\em CVPR}, 2015.

\bibitem{mao2014explain}
Junhua Mao, Wei Xu, Yi~Yang, Jiang Wang, and Alan~L. Yuille.
\newblock Explain images with multimodal recurrent neural networks.
\newblock {\em ArXiv}, 2014.

\bibitem{vinyals2014show}
Oriol Vinyals, Alexander Toshev, Samy Bengio, and Dumitru Erhan.
\newblock Show and tell: A neural image caption generator.
\newblock {\em CVPR}, 2014.

\bibitem{MSR}
H.~Fang, S.~Gupta, F.~N. Iandola, R.K Srivastava, L.~Deng, P.~Dollár, J.~Gao,
  X.~He, Margaret. Mitchell, J.~C. Platt, C.~L. Zitnick, and G.~Zweig.
\newblock From captions to visual concepts and back.
\newblock In {\em CVPR}, 2015.

\bibitem{Klein}
Benjamin Klein, Guy Lev, Gil Sadeh, and Lior Wolf.
\newblock Associating neural word embeddings with deep image representations
  using fisher vectors.
\newblock In {\em CVPR}, 2015.

\bibitem{skipthoughts}
Ryan Kiros, Yukun Zhu, Ruslan Salakhutdinov, Richard~S. Zemel, Antonio
  Torralba, Raquel Urtasun, and Sanja Fidler.
\newblock Skip-thought vectors.
\newblock {\em NIPS}, 2015.

\bibitem{Gong}
Y.~Gong, L.~Wang, M.~Hodosh, J.~Hockenmaier, and S.~Lazebnik.
\newblock Improving image-sentence embeddings using large weakly annotated
  photo collections.
\newblock In {\em ECCV}, 2014.

\bibitem{Socher_zero}
Richard Socher, Milind Ganjoo, Hamsa Sridhar, Osbert Bastani, Christopher~D.
  Manning, and Andrew~Y. Ng.
\newblock Zero-shot learning through cross-modal transfer.
\newblock {\em NIPS}, 2013.

\bibitem{RegularizedCCA}
H.D. Vinod.
\newblock Canonical ridge and econometrics of joint production.
\newblock {\em Journal of Econometrics}, 1976.

\bibitem{Golub73numericalmethods}
Gene~H. Golub, He~Bjlirck, and Gene~H. Golub.
\newblock Numerical methods for computing angles between linear subspaces.
\newblock {\em Math. Comp}, 1973.

\bibitem{T-SVD}
Per~C. Hansen.
\newblock The truncated svd as a method for regularization.
\newblock {\em Technical report}, 1986.

\bibitem{MSCOCO}
Tsung{-}Yi Lin, Michael Maire, Serge~J. Belongie, Lubomir~D. Bourdev, Ross~B.
  Girshick, James Hays, Pietro Perona, Deva Ramanan, Piotr Doll{\'{a}}r, and
  C.~Lawrence Zitnick.
\newblock Microsoft {COCO:} common objects in context.
\newblock {\em EECV}, 2014.

\bibitem{Flickr8k}
Micah Hodosh, Peter Young, and Julia Hockenmaier.
\newblock Framing image description as a ranking task: Data, models and
  evaluation metrics.
\newblock {\em J. Artif. Int. Res.}, 2013.

\bibitem{Flickr30k}
Peter Young, Alice Lai, Micah Hodosh, and Julia Hockenmaier.
\newblock From image descriptions to visual denotations: New similarity metrics
  for semantic inference over event descriptions.
\newblock {\em {TACL}}, 2014.

\bibitem{hotelling1936relations}
Harold Hotteling.
\newblock Relations between two sets of variates.
\newblock {\em Biometrika}, 1936.

\bibitem{Golub1995}
Gene~H. Golub and Hongyuan Zha.
\newblock Springer, 1995.

\bibitem{Relevance}
John Lafferty and ChengXiang Zhai.
\newblock {\em Language Modeling for Information Retrieval}.
\newblock Springer Netherlands, 2003.

\bibitem{word2vec}
Tomas Mikolov, Kai Chen, Greg Corrado, and Jeffrey Dean.
\newblock Efficient estimation of word representations in vector space.
\newblock {\em ArXiv}, 2013.

\bibitem{kmeanemb1}
Alex Smola, Arthur Gretton, Le~Song, and Bernhard Schölkopf.
\newblock A hilbert space embedding for distributions.
\newblock Springer, 2007.

\bibitem{BagWords1}
Yuya Yoshikawa, Tomoharu Iwata, Hiroshi Sawada, and Takeshi Yamada.
\newblock Cross-domain matching for bag-of-words data via kernel embeddings of
  latent distributions.
\newblock In {\em NIPS}. 2015.

\bibitem{UnivKernel}
Andreas Christmann and Ingo Steinwart.
\newblock Universal kernels on non-standard input spaces.
\newblock In {\em NIPS}. 2010.

\bibitem{RF}
Ali Rahimi and Benjamin Recht.
\newblock Random features for large-scale kernel machines.
\newblock In {\em NIPS}. 2008.

\bibitem{vgg}
Karen Simonyan and Andrew Zisserman.
\newblock Very deep convolutional networks for large-scale image recognition.
\newblock {\em ICLR}, 2015.

\bibitem{vendrov2015order}
Ivan Vendrov, Ryan Kiros, Sanja Fidler, and Raquel Urtasun.
\newblock Order-embeddings of images and language.
\newblock {\em ICLR}, 2015.

\bibitem{WangLL15}
Liwei Wang, Yin Li, and Svetlana Lazebnik.
\newblock Learning deep structure-preserving image-text embeddings.
\newblock {\em CVPR}, 2016.

\bibitem{he15deepresidual}
Kaiming He, Xiangyu Zhang, Shaoqing Ren, and Jian Sun.
\newblock Deep residual learning for image recognition.
\newblock In {\em CVPR}, 2016.

\bibitem{Word2Gauss}
Luke Vilnis and Andrew McCallum.
\newblock Word representations via gaussian embedding.
\newblock {\em ICLR}, 2015.

\bibitem{ProdKernel}
Tony Jebara and Risi Kondor.
\newblock Bhattacharyya and expected likelihood kernels.
\newblock In {\em COLT}, 2003.

\end{thebibliography}
\begin{appendix}
\onecolumn
\begin{center}
\large{\textbf{Supplementary Material \\ Asymmetric Regularized CCA and Hierarchical Kernel Sentence Embedding for Image \& Text Retrieval}}
\end{center}

\section{Appendix: Query Generation Versus Search Generation} \label{ap:Laff}
 \cite{Relevance} propose to use a binary random variable $r$ that denotes relevance between a query and a search item, $r=1$ if there is a match and $0$ otherwise. In order to rank search items  \cite{Relevance} propose the  use  of the log-odds ratio:
 $$\log \frac{\mathbb{P}(r=1|q,s)}{\mathbb{P}(r=0|q,s)}$$
Using the search generation approach this ratio is equivalent to \cite{Relevance}:
$$\log \frac{\mathbb{P}(s|q,r=1)}{\mathbb{P}(s|q,r=0)},$$
hence this approach needs both positive and negative pairs of search and query items (triplet losses are instances of this approach).
On the other hand using the query generation approach (under mild assumptions) this ratio is equivalent to \cite{Relevance}:
$$\log \mathbb{P}(q|s,r=1)$$
The query generation approach models relevance in an implicit way and does not need negative samples.
\section{Appendix: Proofs} \label{ap:proofs}
\begin{proof}[Proof of Lemma 1] We give the proof here as it will be useful in the the development of the full regularization path of CCA with truncated SVD regularization of the covariances $C_{XX}$ and $C_{YY}$.
Let $P_x= \Sigma_xV_x^{\top}U \in \mathbb{R}^{m_x\times k } \text{ equivalently } U =V_x \Sigma_x^{-1}P_x $
and $P_y= \Sigma_yV_y^{\top}V \in \mathbb{R}^{m_y\times k } \text{ equivalently } V=V_y \Sigma_y^{-1}P_y $.
Hence we obtain by this change of variable:
$U^{\top}X^{\top}YV= P^{\top}_x\Sigma^{-1}_xV^{\top}_x V_x\Sigma_x U^{\top}_xU_y \Sigma_yV_y^{\top}V_y\Sigma_y^{-1}P_y= P_x^{\top} \left(U^{\top}_xU_y\right) P_y.$
Similarly:
$U^{\top}X^{\top}X U= P^{\top}_{x}P_x ~~~ V^{\top}Y^{\top}Y V= P^{\top}_{y}P_y$.
Hence replacing $U,V$ with $P_x,P_y$ we have:
$$\max_{P^{\top}_xP_x=I, P^{\top}_yP_y=I} Tr(P_x^{\top}\left(U^{\top}_xU_y\right)P_y),$$
This is solved by an SVD of $ T=U^{\top}_xU_y$.
$[P_x, \Sigma,P_y]=SVD(T)$, $(P_{x}\in \mathbb{R}^{m_x\times k}, \Sigma \in \mathbb{R}^{k\times k},P_{y}\in \mathbb{R}^{m_y\times k}$).
where $k=\min(m_x,m_y)$,
and finally we have $U =V_x \Sigma_x^{-1}P_x $, $ V=V_y \Sigma_y^{-1}P_y $. 
\end{proof}
\begin{proof}[Proof of Theorem 1] \label{ap:proofReg} 
~\\
1) Tikhonov:
\begin{equation}
\max_{U^{\top}(X^{\top}X+\gamma_x I_{m_x})U=I, V^{\top}(Y^{\top}Y+\gamma_y I_{m_y})V=I }Tr(U^{\top}X^{\top}YV).
\label{eq:rtikh1}
\end{equation}
$[U_x,\Sigma_x,V_x]=SVD(X)~~U_x \in \mathbb{R}^{n\times m_x}, \Sigma_x \in \mathbb{R}^{m_x \times m_x} , V_x \in \mathbb{R}^{m_x \times m_x} ~~ X= U_x \Sigma_x V^{\top}_x.$
$[U_y,\Sigma_y,V_y]=SVD(Y)~~ U_y \in \mathbb{R}^{n\times m_y}, \Sigma_x \in \mathbb{R}^{m_y \times m_y} , V_x \in \mathbb{R}^{m_y \times m_y} ~~Y= U_y \Sigma_y V^{\top}_y.$
$$X^{\top}X+\gamma_x I = V_{x} \left(\Sigma^2_x+\gamma_x I\right) V^{\top}_{x}. $$
$$Y^{\top}Y+\gamma_y I =V_{y} \left( \Sigma^2_y+\gamma_y I\right)V^{\top}_{y}.$$
Let $P_x= \sqrt{\Sigma^2_x+\gamma_x I} V_x^{\top}U \in \mathbb{R}^{m_x\times k } \text{ equivalently } U =V_x  \left(\Sigma^2_x+\gamma_x I\right)^{-\frac{1}{2}}P_x $.
Let $P_y= \sqrt{\Sigma^2_y+\gamma_y I} V_y^{\top}V \in \mathbb{R}^{m_y\times k } \text{ equivalently } V =V_y  \left(\Sigma^2_y+\gamma_y I\right)^{-\frac{1}{2}}P_y $.
Hence we obtain by this change of variable for  the objective in \eqref{eq:rtikh1}:
\begin{eqnarray*}
U^{\top}X^{\top}YV&=& P^{\top}_x   \left(\Sigma^2_x+\gamma_x I\right)^{-\frac{1}{2}}V^{\top}_x V_x\Sigma_x U^{\top}_xU_y \Sigma_yV_y^{\top}V_y  \left(\Sigma^2_y+\gamma_y I\right)^{-\frac{1}{2}}P_y\\
&=& P_x^{\top}\left(\Sigma^2_x+\gamma_x I\right)^{-\frac{1}{2}}\Sigma_x \left(U^{\top}_xU_y\right)  \Sigma_y  \left(\Sigma^2_y+\gamma_y I\right)^{-\frac{1}{2}}P_y.
\end{eqnarray*}
Let $$T_{\gamma_x,\gamma_y}= \left(\Sigma^2_x+\gamma_x I\right)^{-\frac{1}{2}}\Sigma_x \left(U^{\top}_xU_y\right)  \Sigma_y  \left(\Sigma^2_y+\gamma_y I\right)^{-\frac{1}{2}},$$
hence we have:
\begin{equation*}
U^{\top}X^{\top}YV = P_{x}^{\top}T_{\gamma_x,\gamma_y}P_y.
\end{equation*}
On the other hand, plugging this change of variable in the constraints of \eqref{eq:rtikh1} we obtain:
$$U^{\top}(X^{\top}X+\gamma_x I)U =U^{\top} V_{x} \left(\Sigma^2_x+\gamma_x I\right) V^{\top}_{x}U= P^{\top}_{x}P_{x} =I $$
$$V^{\top}(Y^{\top}Y+\gamma_y I)V =V^{\top} V_{y} \left(\Sigma^2_y+\gamma_y I\right) V^{\top}_{y}V= P^{\top}_{y}P_{y} =I $$
Therefore using this change of variable, problem \eqref{eq:rtikh1} becomes:
\begin{equation}
\max_{P_{x}^{\top}P_x=I, P_{y}^{\top}P_y=I}Tr(P_{x}^{\top}T_{\gamma_x,\gamma_y}P_y),
\end{equation}
this is the variational formulation of the SVD of $T_{\gamma_x,\gamma_y}$.
Hence we obtain that:
$$[P_x,\Sigma,P_y]=SVD(T_{\gamma_x,\gamma_y}),$$
$$U =V_x  \left(\Sigma^2_x+\gamma_x I\right)^{-\frac{1}{2}}P_x ,$$
$$V =V_y  \left(\Sigma^2_y+\gamma_y I\right)^{-\frac{1}{2}}P_y.$$
2) T-SVD:
\begin{equation}
\max_{U^{\top}X^{\top}_{k_x}X_{k_x}U=I, V^{\top}Y_{k_y}^{\top}Y_{k_y}V=I }Tr(U^{\top}X^{\top}YV)
\label{eq:rT-SVD1}
\end{equation}
Let $$P_{x}= \Sigma_{k_x}V_{k_x}^{\top}U \in \mathbb{R}^{k_x\times k},~~ \text{ equivalently } U =V_{k_x} \Sigma_{k_x}^{-1}P_{x} \in \mathbb{R}^{m_x\times k}, k=\min(k_x,k_y) $$
 $$P_{y}= \Sigma_{k_y}V_{k_y}^{\top}V \in \mathbb{R}^{k_y\times k},~~ \text{ equivalently } V=V_{k_y} \Sigma_{k_y}^{-1}P_{y} \in \mathbb{R}^{m_y\times k},  k=\min(k_x,k_y) $$
Hence we obtain by this change of variable, in the objective of \eqref{eq:rT-SVD1}:
$$U^{\top}X^{\top}YV= P^{\top}_{x}\Sigma^{-1}_{k_x}V^{\top}_{k_x} V_x\Sigma_x U^{\top}_xU_y \Sigma_yV_y^{\top}V_{K_y}\Sigma_{k_y}^{-1}P_{y},$$
Now we turn to:
\begin{eqnarray*}
\Sigma^{-1}_{k_x}(V^{\top}_{k_x} V_x)\Sigma_x U^{\top}_x &=&  \Sigma^{-1}_{k_x} [I_{k_x\times k_x} 0_{k_{x}\times(m_x-k_x) } ] \Sigma_x U^{\top}_x\\
&=& \Sigma^{-1}_{k_x} [\Sigma_{k_x} 0_{k_{x}\times(m_{x}-k_{x})} ] U^{\top}_x\\
&=&[I_{k_x\times k_x}  0_{k_{x}\times(m_x-k_x)}] U^{\top}_x \\
&=&U^{\top}_{k_{x}}
\end{eqnarray*}
Hence we keep the first $k_x$ columns of $U_x$, that is $U_{k_x}$. The same argument hold for $U_{k_y}$. It follows that using truncated SVD, we have:
$$U^{\top}X^{\top}YV = P_{x}^{\top} U^{\top}_{k_x}U_{k_y} P_{y}.$$
Let $$T_{k_x,k_y}=U^{\top}_{k_x}U_{k_y}, $$
then using this change of variable, the objective in \eqref{eq:rT-SVD1} becomes:
$$U^{\top}X^{\top}YV = P_{x}^{\top} T_{k_x,k_y} P_{y}.$$
Now turning to the constraints of  \eqref{eq:rT-SVD1}, using this change of variable we obtain:
$$ U^{\top}X^{\top}_{k_x}X_{k_x}U= U^{\top} V_{k_x}\Sigma^2_{k_x}V_{x}^{\top}U=P_x^{\top}P_{x}=I,$$
$$ V^{\top}Y^{\top}_{k_y}Y_{k_y}V= V^{\top} V_{k_y}\Sigma^2_{k_y}V_{y}^{\top}V=P_y^{\top}P_{y}=I.$$
Therefore using this change of variable, problem \eqref{eq:rT-SVD1} becomes:
\begin{equation}
\max_{P_{x}^{\top}P_x=I, P_{y}^{\top}P_y=I}Tr(P_{x}^{\top}T_{k_x,k_y}P_y),
\end{equation}
this is the variational formulation of the SVD of $T_{k_x,k_y}$.
Hence truncated SVD-CCA can be solved finding: 
$$[P_{k_x},\Sigma^{k_x,k_y},P_{k_y}]=SVD(T_{k_x,k_y}).$$
 Turning now to $U^{\top}_{k_x}U_{k_y}$ this can be computed efficiently by precomputing $T=U^{\top}_xU_y \in \mathbb{R}^{m_x\times m_y}$ and then extracting the submatrix consisting of $k_x$ rows and $k_y $ columns.
we return therefore $U =V_{k_x} \Sigma_{k_x}^{-1}P_{x} $, $ V=V_{k_y} \Sigma_{k_y}^{-1}P_{y} $.
\end{proof}

\section{Algorithms} \label{ap:algorithms}

%\setlength{\textfloatsep}{-0.0004em}
%\begin{algorithm}[t!]
%\begin{algorithmic}[1]
%  \State {\bf Input:} dataset $S \in \mathcal{R}^{N \times d}$, initial
%centroids $C \in \mathcal{R}^{k \times d}$.
%  \State {\bf Output:} converged centroids $C$.
%  \medskip
%  \State $\mathscr{T} \gets$ a tree built on $S$
%  \WHILE{centroids $C$ not converged}
%    \State \COMMENT{Remove nodes in the tree if possible.}
%    \State $\mathscr{T} \gets \mathtt{CoalesceNodes(}\mathscr{T}\mathtt{)}$
%    \State $\mathscr{T}_c \gets$ a tree built on $C$
%    \medskip
%    \State \COMMENT{Call dual-tree algorithm.}
%    \State Perform a dual-tree recursion with $\mathscr{T}$, $\mathscr{T}_c$,
%\texttt{BaseCase()}, and \texttt{Score()}.
%    \medskip
%    \State \COMMENT{Restore the tree to its non-coalesced form.}
%    \State $\mathscr{T} \gets \mathtt{DecoalesceNodes(\mathscr{T})}$
%    \medskip
%    \State \COMMENT{Update centroids and bounding information.}
%    \State $C \gets \mathtt{UpdateCentroids(}\mathscr{T}\mathtt{)}$
%    \State $\mathscr{T} \gets \mathtt{UpdateTree(}\mathscr{T}\mathtt{)}$
%  \ENDWHILE
%  \State {\bf return} $C$
%\end{algorithmic}
%\caption{High-level outline of dual-tree $k$-means.}
%\label{alg:high_level}
%\end{algorithm}
%
%\twocolumn
\begin{algorithm}[H]
%\algblock[Bjorck Golub Algorithm]{Start}{End}
 \begin{algorithmic}[1]
%{ }{$X \in \mathbb{R}^{n \times m_x},Y \in \mathbb{R}^{n \times m_y}$}
 \State  $[U_x,\Sigma_x,V_x]= SVD(X)$.
 \State  $[U_y,\Sigma_y,V_y]= SVD(Y)$.
 \State$T = U^{\top}_{x}U_y \in \mathbb{R}^{m_x\times m_y}$ 
 \State $[P_x,\Sigma,P_y]= SVD(T)$
 \State $U =V_x \Sigma_x^{-1}P_x $
 \State $V=V_y \Sigma_y^{-1}P_y $
  \State {return} $U,V$ % \EndProcedure
 \end{algorithmic}
 \caption{Bjorck Golub }
 \label{ALG:BjorckGolub}
\end{algorithm}
\vspace{-0.005 in}
\begin{algorithm}[H]
 \begin{algorithmic}[1]
 %\Procedure{Tikhonov Regularized CCA}{$X \in \mathbb{R}^{n \times m_x},Y \in \mathbb{R}^{n \times m_y}$}
 \State  $[U_x,\Sigma_x,V_x]= SVD(X)$.
  \State  $[U_y,\Sigma_y,V_y]= SVD(Y)$.
  \State $T_0 = \Sigma_x(U^{\top}_{x}U_y )\Sigma_y\in \mathbb{R}^{m_x\times m_y}$ 
 \For {$\gamma_x \in \{\sigma^{2}_{x,1},\dots \sigma^{2}_{x,m_{x}}\}$}\State \Comment {The set of singular values squared or a subsampled grid}
  \For {$\gamma_y \in  \{\sigma^{2}_{y,1},\dots \sigma^{2}_{y,m_{y}}\}$} %\Comment{$\Delta_y$ is the step size in this grid}
  \State $T= \left(\Sigma^2_x+\gamma_x I\right)^{-\frac{1}{2}}T_0 \left(\Sigma^2_y+\gamma_y I\right)^{-\frac{1}{2}}$ 
  \State $[P_x,\Sigma^{\gamma_x,\gamma_y},P_y]= SVD(T)$
  \State $U =V_x  \left(\Sigma^2_x+\gamma_x I\right)^{-\frac{1}{2}}P_x  $
  \State $V=V_y  \left(\Sigma^2_y+\gamma_y I\right)^{-\frac{1}{2}}P_y$
   \State Compute  performance using $U,V,\Sigma^{\gamma_x,\gamma_y}$ on a validation Set. 
  \State (Bidirectional Retrieval is done using  a sorted list of the scores of AW-CCA) 
 \EndFor
 \EndFor
 \State {return} $U,V,\Sigma^{\gamma_x,\gamma_y},\gamma_x,\gamma_y$ with best validation performance for each task. 
 %\EndProcedure
 \end{algorithmic}
 \caption{Tikhonov Regularized CCA (X,Y)}
 \label{ALG:Tikhonovcca}
\end{algorithm}
\begin{algorithm}[H]
 \begin{algorithmic}[1]
% \Procedure{Cross validation truncated SVD -CCA }{$X \in \mathbb{R}^{n \times m_x},Y \in \mathbb{R}^{n \times m_y}$}
  \State $[U_x,\Sigma_x,V_x]= SVD(X)$.
  \State $[U_y,\Sigma_y,V_y]= SVD(Y)$.
   \State$T = U^{\top}_{x}U_y \in \mathbb{R}^{m_x\times m_y}$ 
  \State$W^x = V_x \Sigma_x^{-1} \in \mathbb{R}^{m_x\times m_x}$
  \State $W^y= V_y \Sigma_{y}^{-1}\in \mathbb{R}^{m_y\times m_y} $
 \For {$k_x \in [m_x]$}\State \Comment{$[m_x]=\{1\dots m_x\}$ or a subsampled grid.}
  \For{$k_y \in [m_y]$}
  \State $[P_{x},\Sigma^{k_x,k_y},P_{y}]= SVD(T_{1:k_x,1:k_y})$ \State \Comment{$T_{1:k_x,1:k_y}$: extracts the first $k_x$ rows ,and the first $k_y$ columns of $T$.}
  \State $U =W^x_{:,1:k_x}P_{x} $
  \State $V=W^y_{:,1:k_y} P_{y} $
  \State Compute  performance using $U,V,\Sigma^{k_x,k_y}$ on a validation Set. 
  \State(Bidirectional Retrieval is done using  a sorted list of the scores of  AW-CCA.)  
  \EndFor
 \EndFor
  \State {return} $U,V,\Sigma^{k_x,k_y},k_x,k_y$ with best validation performance for each task.
 %\EndProcedure
 \end{algorithmic}
 \caption{Truncated SVD CCA (X,Y) }
 \label{ALG:TruncatedSVDcca}
\end{algorithm}

\begin{algorithm}[h]
 \begin{algorithmic}[1]
% \Procedure{Guided Tikhonov Validation by T-SVD }{$X \in \mathbb{R}^{n \times m_x},Y \in \mathbb{R}^{n \times m_y}$}
   \State $(k^*_x,k^*_y)= \text{ Truncated SVD CCA }(X,Y)$
  \State $(\gamma_x,\gamma_y)=(\sigma^2_{x,k^*_x},\sigma^2_{y,k^*_y} )$
 %\State $T= \left(\Sigma^2_x+\gamma_x I\right)^{-\frac{1}{2}}\Sigma_x(U^{\top}_{x}U_y )\Sigma_y \left(\Sigma^2_y+\gamma_y I\right)^{-\frac{1}{2}}$ 
  \State $[P_x,\Sigma^{\gamma_x,\gamma_y},P_y]= SVD(T_{\gamma_x,\gamma_y})$
  \State $U =V_x  \left(\Sigma^2_x+\gamma_x I\right)^{-\frac{1}{2}}P_x $, 
  \State $V=V_y  \left(\Sigma^2_y+\gamma_y I\right)^{-\frac{1}{2}}P_y$
  \State {return} $U,V,\Sigma^{\gamma_x,\gamma_y}$.
% \EndProcedure
 \end{algorithmic}
 \caption{Guided Tikhonov Validation by T-SVD(X,Y) }
 \label{ALG:GTikhonovcca}
\end{algorithm}

\section{Query Generation versus Search Generation versus CCA}\label{sec:tsne}
When we retrieve images matching a text query. Three methods are possible:
\begin{enumerate}
\item CCA: We cross validate from image search using the cosine $\frac{\scalT{U^{\top}x}{V^{\top}y}}{\nor{U^{\top}x}\nor{V^{\top}y}}$. Then we retrieve for the test query images with largest cosine.
\item Query Generation Approach: Where we weigh the canonical weights of search space (images) that is $U$, we cross-validate using the score $\frac{\scalT{\Sigma U^{\top}x}{V^{\top}y}}{\nor{\Sigma U^{\top}x}\nor{V^{\top}y}}$. We retrieve for the query $y$ images having largest score as defined here.
\item Search Generation Approach: Where we weigh the canonical weights of query space (text ) that is $V$,  we cross-validate using the score $\frac{\scalT{ U^{\top}x}{\Sigma V^{\top}y}}{\nor{U^{\top}x}\nor{\Sigma V^{\top}y}}$
\end{enumerate} 

To illustrate this we consider the following text query:" A kitchen with two windows and two metals sinks". We embed this sentence using HKSE(rbf,rbf) get a vector $y$.
We do three experiments:
\begin{enumerate}
\item We embed the image test set with $\Sigma U^{\top}$, and the query with $V^{\top}$, this corresponds to the Query generation approach. Then we embed the query and the nearest neighbor images in 2D using the t-SNE plot. We represent the text query with the ground truth image with a blue box around it.
\item We embed the image test set with $ U^{\top}$, and the query with $V^{\top}$, this corresponds to CCA. Then we embed the query and the nearest neighbor images in 2D using the t-SNE plot. We represent the text query with the ground truth image with a blue box around it.
\item We embed the image test set with $ U^{\top}$, and the query with $\Sigma V^{\top}$, this corresponds to the search generation approach. Then we embed the query and the nearest neighbor images in 2D using the t-SNE plot. We represent the text query with the ground truth  image with a blue box around it.
\end{enumerate}

We see in Figure \ref{fig:tsne} that the space is better organized in the query generation approach.
   \begin{figure}[ht]
%\hspace*{-0.3in}
  \begin{subfigure}[t]{0.6\textwidth}     
       %\centering   
        \includegraphics[scale=0.3]{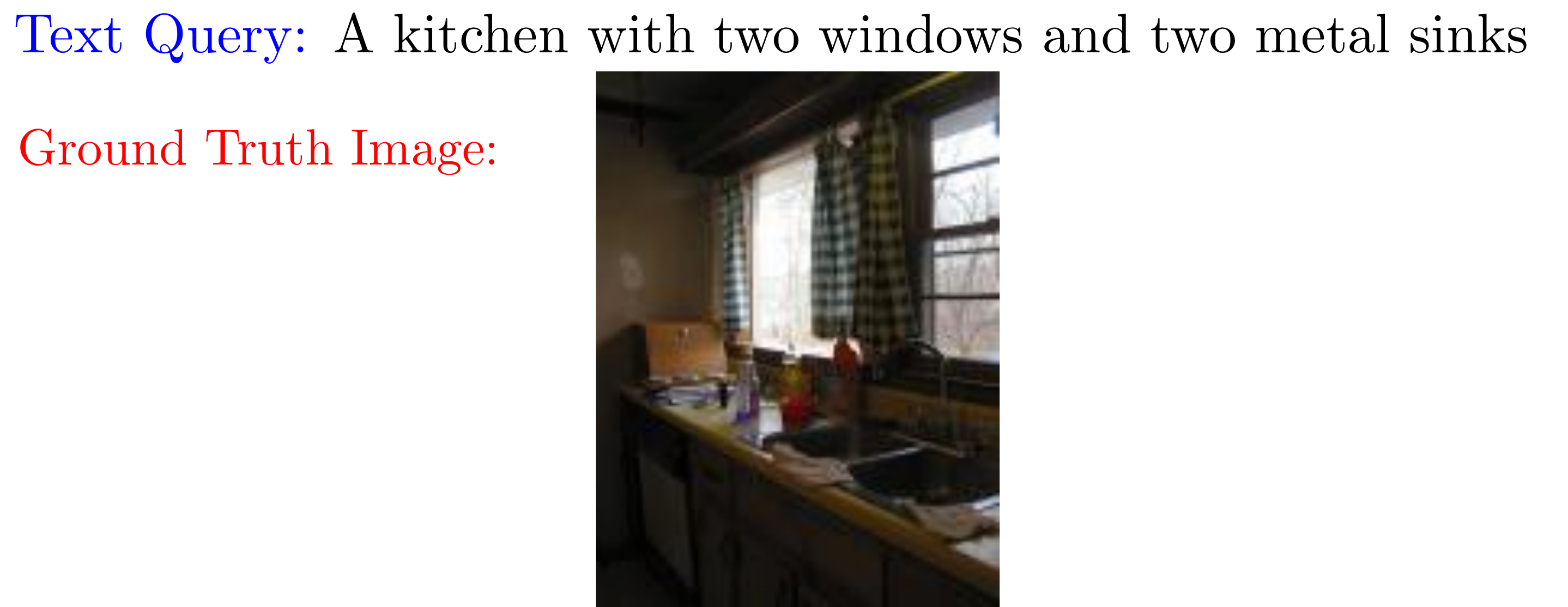}
        % \caption{T-SVD Cross Validation.} %\\  \textcolor{red}{TO BE REPLACED HERE}}
    \end{subfigure}
~
\hspace*{-0.8in}
    \begin{subfigure}[t]{0.6\textwidth} 
      % \centering
            \includegraphics[scale=0.3]{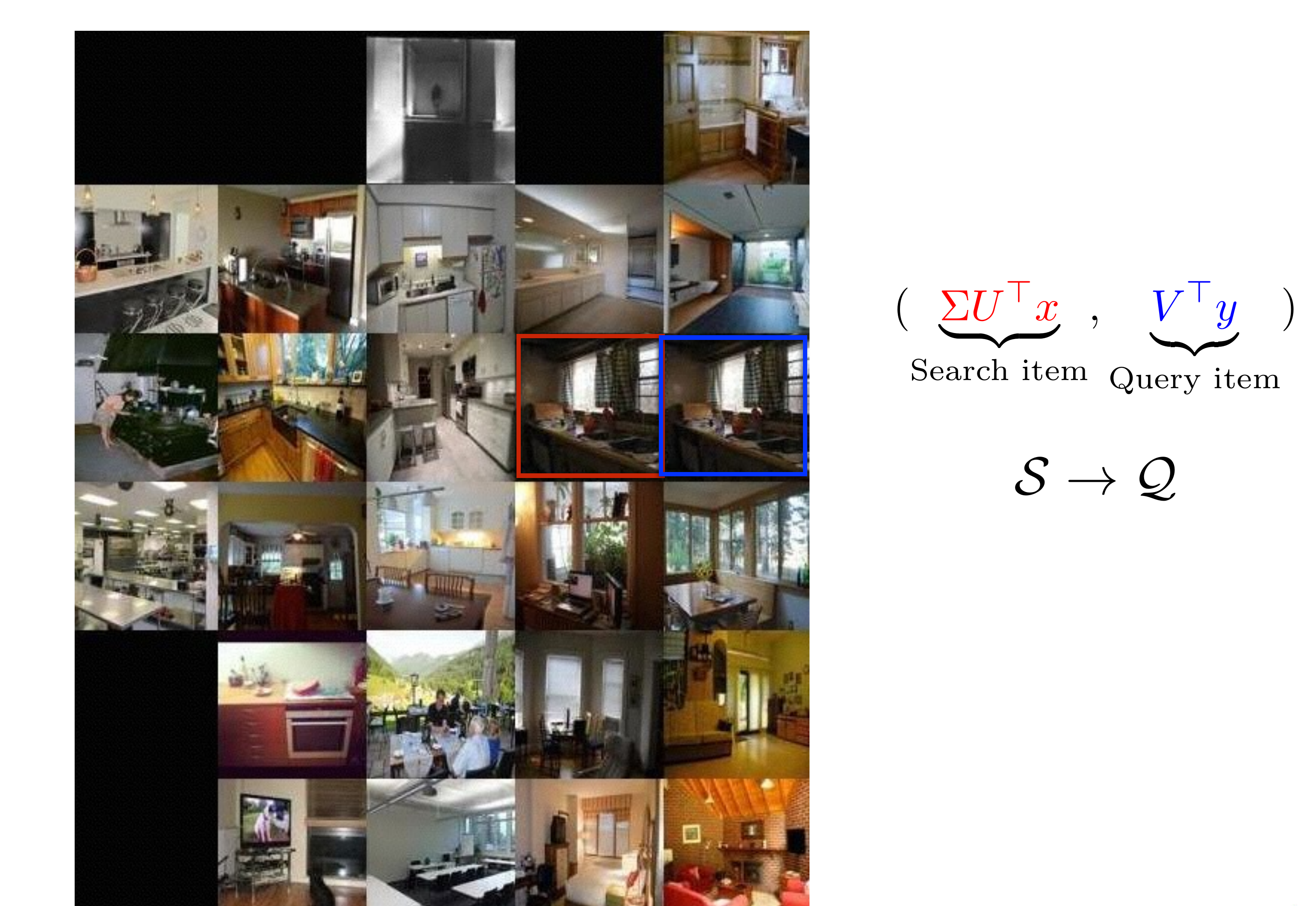}
         
             \caption{Query Generation Approach.}.% \\\textcolor{red}{TO BE REPLACED HERE}}
    \end{subfigure}
  ~  
%\hspace*{-0.8 in }
    \begin{subfigure}[t]{0.6\textwidth} 
      % \centering
            \includegraphics[scale=0.3]{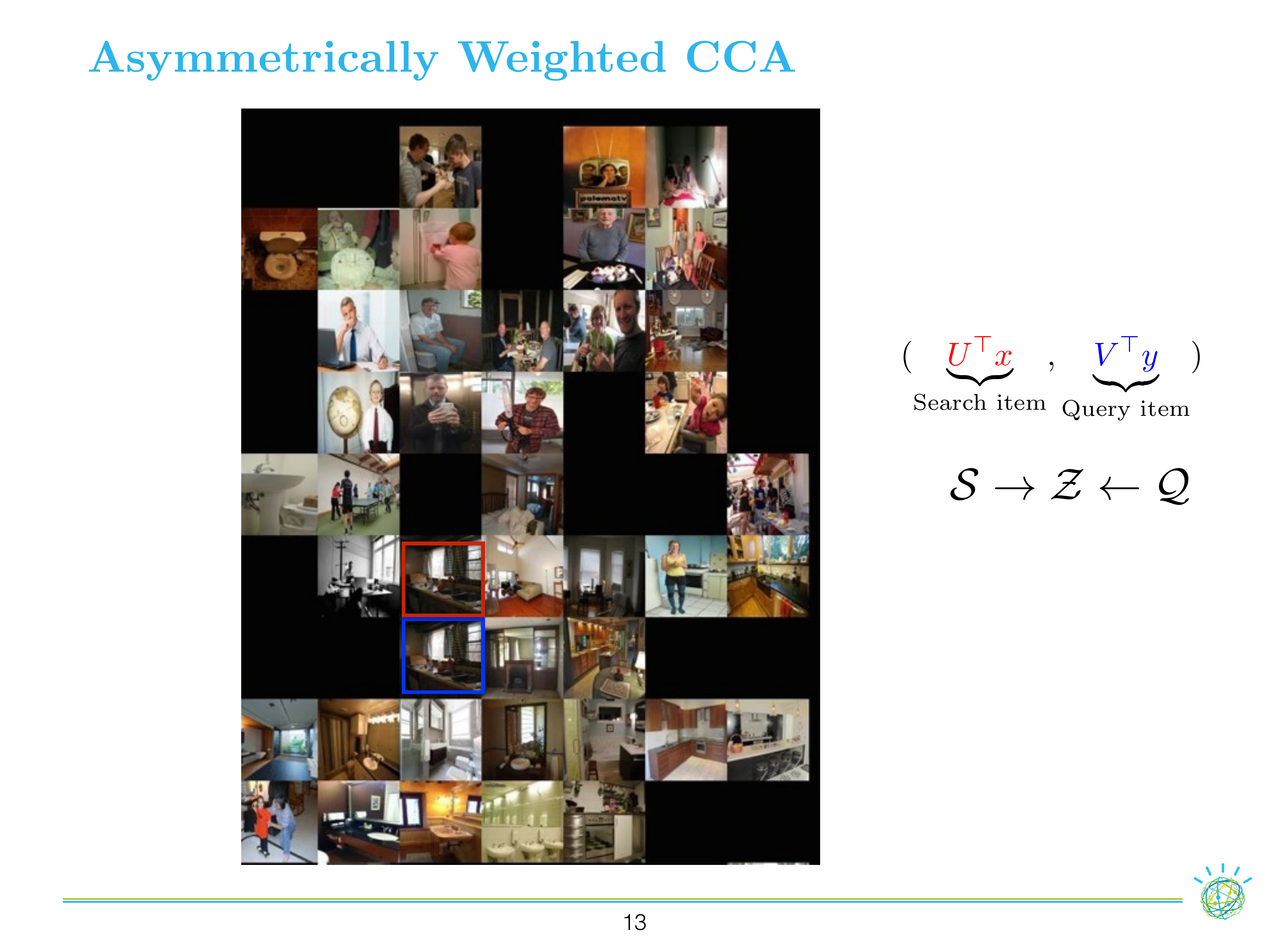}
         
             \caption{CCA.}.% \\\textcolor{red}{TO BE REPLACED HERE}}
    \end{subfigure}
  ~  
    \hspace*{-0.7in}
    \begin{subfigure}[t]{0.6\textwidth} 
      % \centering
            \includegraphics[scale=0.3]{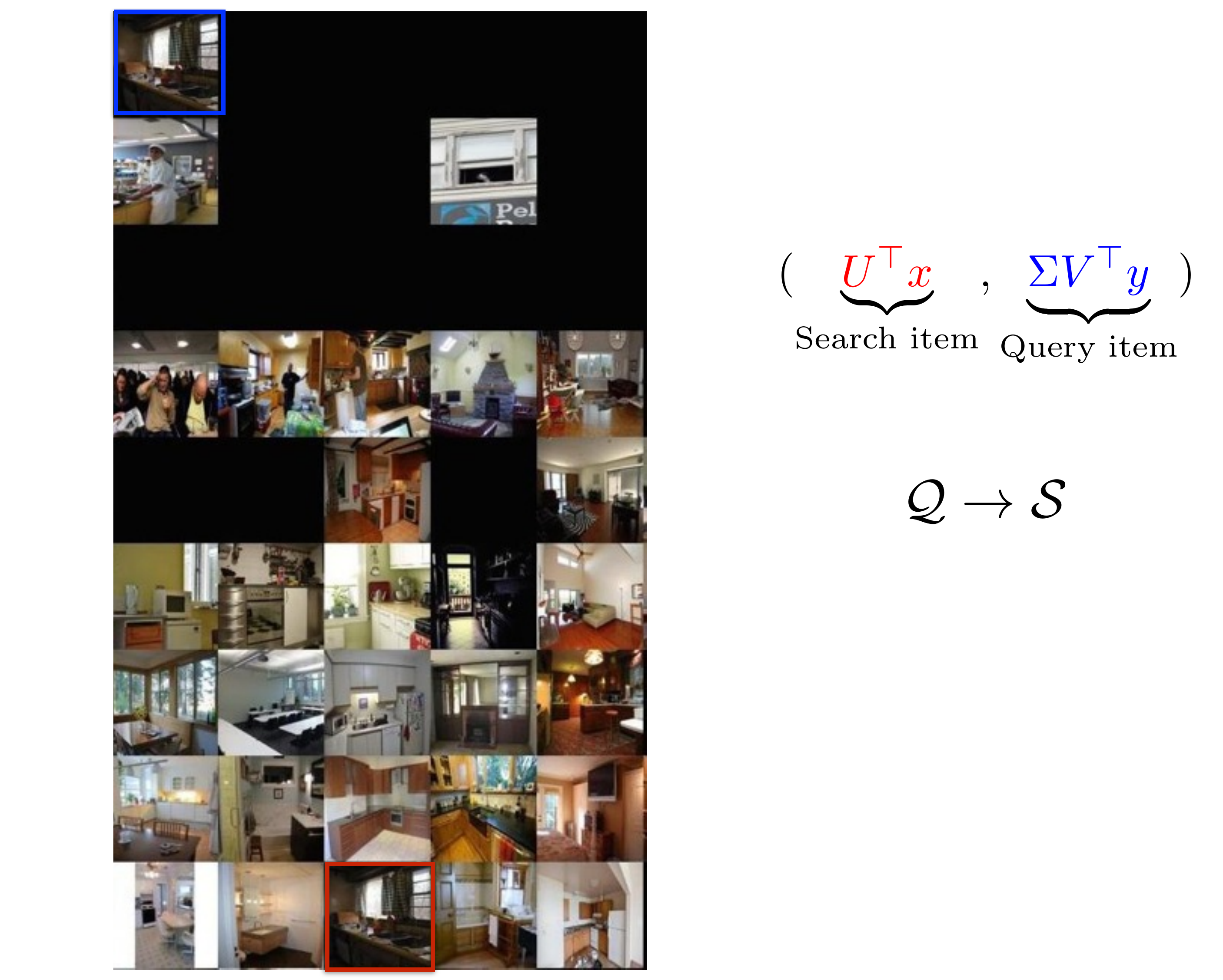}
         
             \caption{Search Generation Approach.}.% \\\textcolor{red}{TO BE REPLACED HERE}}
    \end{subfigure}

         \caption{TSNE plots of different Embeddings.}
             \label{fig:tsne}
  \end{figure}

\section{HKSE } 
\subsection{Proof of Proposition 1}\label{app:prop1}
\begin{proposition} [HKSE approximation] Let $\mathcal{A}$ be the vocabulary embedded in a vector space. Let $|\mathcal{A}|$ be the size of the vocabulary. Let $s$ be the maximum sentence length. Let $\hat{\mu}_{n_1}(\rho_1)= \frac{1}{n_1}\sum_{i=1}^{n_1}\Phi_{\gamma,d}(a_i)$, and $\hat{\mu}_{n_2}(\rho_2)= \frac{1}{n_2}\sum_{i=1}^{n_2}\Phi_{\gamma,d}(b_i)$, on two different sentences $\rho_1,\rho_2$, with words $\{a_1,\dots a_{n_1}\}$ and $\{b_1,\dots b_{n_2}\}$ respectively. Let $\epsilon,\delta >0$,for $m\geq  \frac{1}{2\delta^2}\log(|\mathcal{A}|^2/\epsilon)$ and $m'\geq  \frac{1}{2\delta^2}\log(|\mathcal{A}|^{2s}/\epsilon)$ we have with probability $1-2\epsilon
$:
$$\exp \frac{-3\eta \delta}{2} K(\rho_1,\rho_2) - \delta  \leq \scalT{\Phi_{\eta,m}\left(\hat{\mu}_{n_1}(\rho_1)\right)}{\Phi_{\eta,m}\left(\hat{\mu}_{n_2}(\rho_2)\right)}\leq \exp \frac{3\eta \delta}{2} K(\rho_1,\rho_2)  +\delta ,$$
Informally for 2 layers the error of estimation is multiplicative and additive , inner word level dimension scales logarithmically with the vocabulary size $m =O(\log(|\mathcal{A}|))$, and the outter dimension scales linearly with the maximum sentence length $s$ an logarithmically with the vocabulary size  $m'=O(s\log(|\mathcal{A}|) )$.

\end{proposition}
\begin{proof} 
Let $\mathcal{A}$ be the vocabulary embedded in a vector space. Let $|\mathcal{A}|$ be the size of the vocabulary. We know from Rahimi and Rech that : 
$$\mathbb{E}\scalT{\Phi_{\gamma,d}(a)}{\Phi_{\gamma,d}(b)}=k_{\gamma,d}(a,b) \text{ forall } a, b \in \mathcal{A}$$
Hence applying Hoeffeding inequality and a union bound on all words in $\mathcal{A}$, we obtain that, the following holds with a probability $1-\epsilon$ for all $a,b \in \mathcal{A}$:
$$k_{\gamma,d}(a,b) -\delta \leq\scalT{\Phi_{\gamma,d}(a)}{\Phi_{\gamma,d}(b)}\leq k_{\gamma,d}(a,b) +\delta,$$
for $m\geq \frac{1}{2\delta^2}\log(|\mathcal{A}|^2/\epsilon) $.\\
Let $\hat{\mu}_{n_1}(\rho_1)= \frac{1}{n_1}\sum_{i=1}^{n_1}\Phi_{\gamma,d}(a_i)$, and $\hat{\mu}_{n_2}(\rho_2)= \frac{1}{n_2}\sum_{i=1}^{n_2}\Phi_{\gamma,d}(b_i)$, on two different sentences.
Assume $s$ is the maximum sentence length, a  bound on the total number of sentence is $|\mathcal{A}|^s$
Applying again a hoeffeding bound, and a union bound on all sentences , we have with a probability $1-\epsilon$
$$k_{\eta,m}(\hat{\mu}_{n_1}(\rho_1),\hat{\mu}_{n_2}(\rho_2)) - \delta  \leq \scalT{\Phi_{\eta,m}\left(\hat{\mu}_{n_1}(\rho_1)\right)}{\Phi_{\eta,m}\left(\hat{\mu}_{n_2}(\rho_2)\right)}\leq k_{\eta,m}(\hat{\mu}_{n_1}(\rho_1),\hat{\mu}_{n_2}(\rho_2)) +\delta ,$$
for $m'\geq \frac{1}{2\delta^2}\log(|\mathcal{A}|^{2s}/\epsilon)$.
Let $\Delta =   \frac{1}{n_1^2} \sum_{i,j=1}^{n_1}k_{\gamma,d}(a_i,a_j)+ \frac{1}{n^2_2}\sum_{i,j=1}^{n_2}k_{\gamma,d}(b_i,b_j)  -\frac{1}{n_1n_2}\sum_{i=1}^{n_1}\sum_{j=1}^{n_2}2k_{\gamma,d}(a_i,b_j)$ . By definition:
$$K(\rho_1,\rho_2)=\exp(-\frac{\eta}{2}\Delta)$$
\begin{align*}
k_{\eta,m}(\hat{\mu}_{n_1}(\rho_1),\hat{\mu}_{n_2}(\rho_2))&=\exp-\frac{\eta}{2} \left(\nor{\hat{\mu}_{n_1}(\rho_1)-\hat{\mu}_{n_2}(\rho_2)}^2_{2} \right)\\
&=\exp-\frac{\eta}{2} \left( \nor{\hat{\mu}_{n_1}(\rho_1)-\hat{\mu}_{n_2}(\rho_2)}^2_{2}- \Delta + \Delta \right)\\
&=K(\rho_1,\rho_2) \exp-\frac{\eta}{2} \left( \nor{\hat{\mu}_{n_1}(\rho_1)-\hat{\mu}_{n_2}(\rho_2)}^2_{2}- \Delta\right)
\end{align*}

Note that $m\geq \frac{1}{2\delta^2}\log(|\mathcal{A}|^2/\epsilon)$, we have with probability $1-\epsilon$,
 $$\left|\nor{\hat{\mu}_{n_1}(\rho_1)-\hat{\mu}_{n_2}(\rho_2)}^2_{2}- \Delta\right| \leq 3\delta $$
Hence with probability $1-\epsilon$ 
$$ \exp \frac{-3\eta \delta}{2} K(\rho_1,\rho_2) \leq k_{\eta,m}(\hat{\mu}_{n_1}(\rho_1),\hat{\mu}_{n_2}(\rho_2)) \leq \exp \frac{3\eta \delta}{2} K(\rho_1,\rho_2) $$
Hence for $m\geq  \frac{1}{2\delta^2}\log(|\mathcal{A}|^2/\epsilon)$ and $m'\geq  \frac{1}{2\delta^2}\log(|\mathcal{A}|^{2s}/\epsilon)$ we have with probability $1-2\epsilon
$:
$$\exp \frac{-3\eta \delta}{2} K(\rho_1,\rho_2) - \delta  \leq \scalT{\Phi_{\eta,m}\left(\hat{\mu}_{n_1}(\rho_1)\right)}{\Phi_{\eta,m}\left(\hat{\mu}_{n_2}(\rho_2)\right)}\leq \exp \frac{3\eta \delta}{2} K(\rho_1,\rho_2)  +\delta ,$$
Informally for 2 layers the error of estimation is multiplicative and additive, the inner word level dimension scales logarithmically with the vocabulary size $m =O(\log(|\mathcal{A}|))$, and the outer dimension scales linearly with the maximum sentence length $s$ an logarithmically with the vocabulary size  $m'=O(s\log(|\mathcal{A}|) )$.
\end{proof}
\subsection{HKSE as a Gaussian Embedding of Sentences}\label{app:Gauss2Vec}
 \cite{Word2Gauss} introduced the embedding of words to Gaussian distributions (Word2Gauss), we show that similarly HKSE(lin,rbf) for $k^w$ being linear and $k^{s}$ being an rbf (radial basis function) kernel defines an embedding of sentences to Gaussian distributions (Sent2Gauss). 
Given $\{a_1\dots a_n\}$, the vector embeddings of words in a sentence, we represent a sentence  as a  gaussian distribution $\rho$ : $ \rho \sim \mathcal{N}(\mu,\sigma^2 I_{d}), ~ \mu=\frac{1}{n}\sum_{i=1}^n a_i, \eta >0 .$
In order to compare two sentences we compare two distributions ${\rho_1}$ and ${\rho_2}$, using product kernel between distributions \cite{ProdKernel}:
$K(\rho_1,\rho_2)= \int \mathcal{N}(x; \mu_1, \sigma^2 I_{d})\mathcal{N}(x; \mu_2, \sigma^2 I_{d}) dx
=\mathcal{N}(0; \mu_1-\mu_2, 2\sigma^2 I_{d} )= (4\pi\sigma^2)^{-\frac{d}{2}}\exp \left(-\frac{1}{4 \sigma^2} \nor{\mu_1- \mu_2}^2\right),$ which corresponds to the kernel computed by HKSE(lin,rbf).
 HKSE(rbf,rbf) corresponds also to a universal kernel between distributions \cite{UnivKernel}.

\section{Data Splits}\label{app:splits}
For MSCOCO the training set contains $113,287$ images, along with $5$ captions each. Similarly to \cite {Klein}, we used the splits from  \cite{Karpathy}, and performed cross-validation on  a validation set of $5K$ images, and tested our models on  a test set of  $5$K images, as well as five 1K  splits of the $5 K$ test images  as in \cite{Klein}. We report for both tasks the recall rate at one result, five results, or ten first results (r@1,5,10), as well as the median rank of the first ground truth retrieval.\\
We follow the experimental protocol of
{\url{github.com/ryankiros/skip-thoughts/blob/master/eval_rank.py}}:
For $5K$ test images, we have $25 K$ query captions, the image search retrieval scoring is  based on the ground truth image of each caption. For image annotation, we have $5K$ query test images, to annotate among $25K$ captions, the scoring returns the caption that has the highest cosine within each  $5$ captions. A similar scoring is done for $1K$ tests.
We did not do any vocabulary pruning and kept all words that in vocabulary of word2vec we ended up with a vocbulary size 21975 words for MSCOCO.

The Flickr 8K training set contains $6K$ images  along with $5$ captions each, the validation and test set contain 1K images each. We use the splits as specified in \cite{Flickr8k}. The Flickr 30K \cite{Flickr30k} training set contains $ 25381$ images along with $5$ captions each, the validation and test sets contain $3K$ images each. We follow  \cite{Karpathy} and use 3 random splits of $1K$ images for test and validation and report average performance over three runs.
We did not do any vocabulary pruning and kept all words that in vocabulary of word2vec we ended up with a vocbulary size 16772 words for Flickr30 K and 7745 for Flickr 8k. 
\section{AWT-SVD CCA Regularization Paths }\label{app:RegPaths}

 \begin{figure*}[h!]
%\hspace*{-0.8in}
   \begin{subfigure}[t]{0.6\textwidth}        
        \includegraphics[height=2.2in]{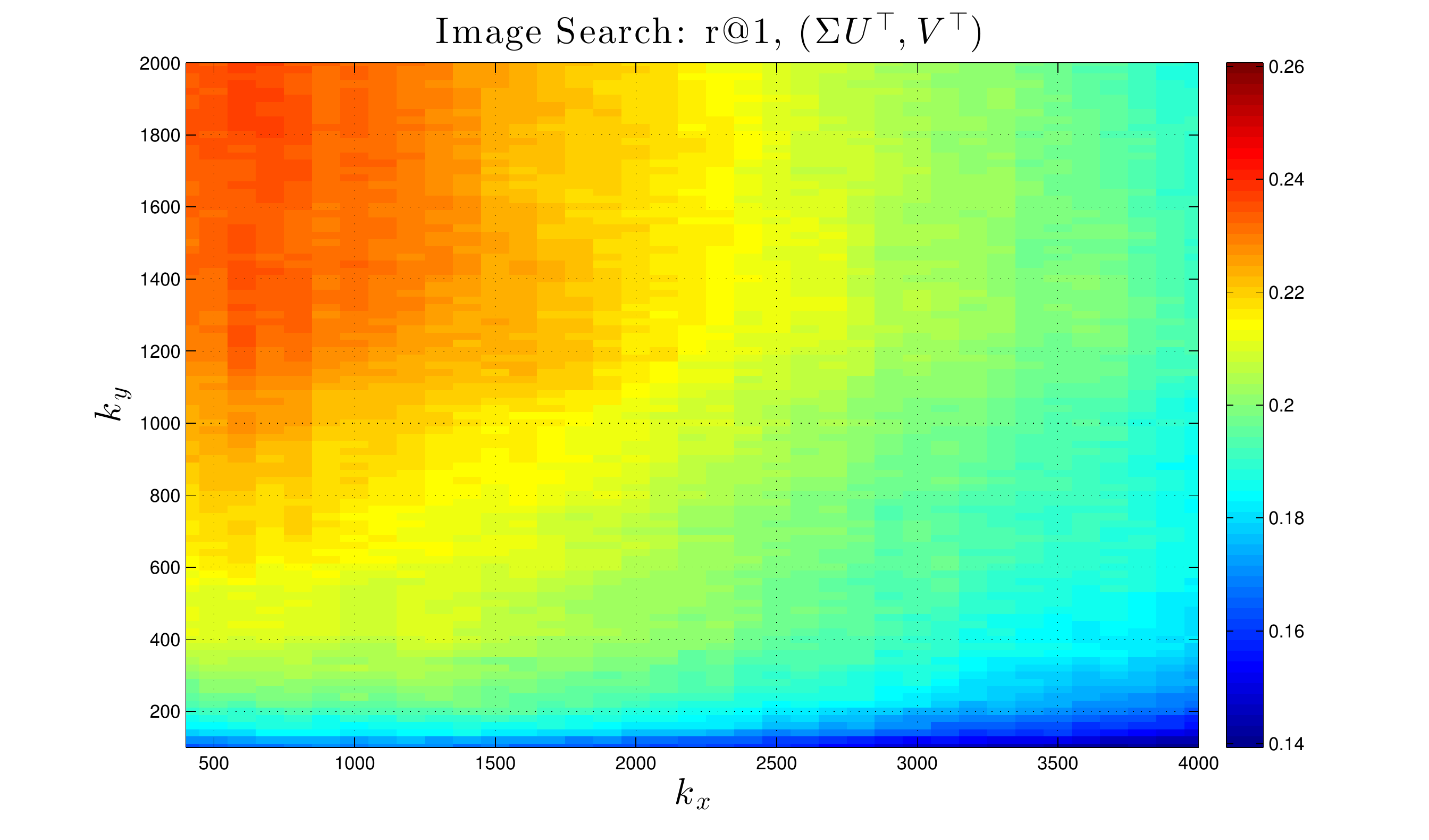}
         \caption{T-SVD Cross Validation.}
    \end{subfigure}
~
\hspace*{-0.6in}
    \begin{subfigure}[t]{0.5\textwidth} 
   
       \includegraphics[height=2.2in]{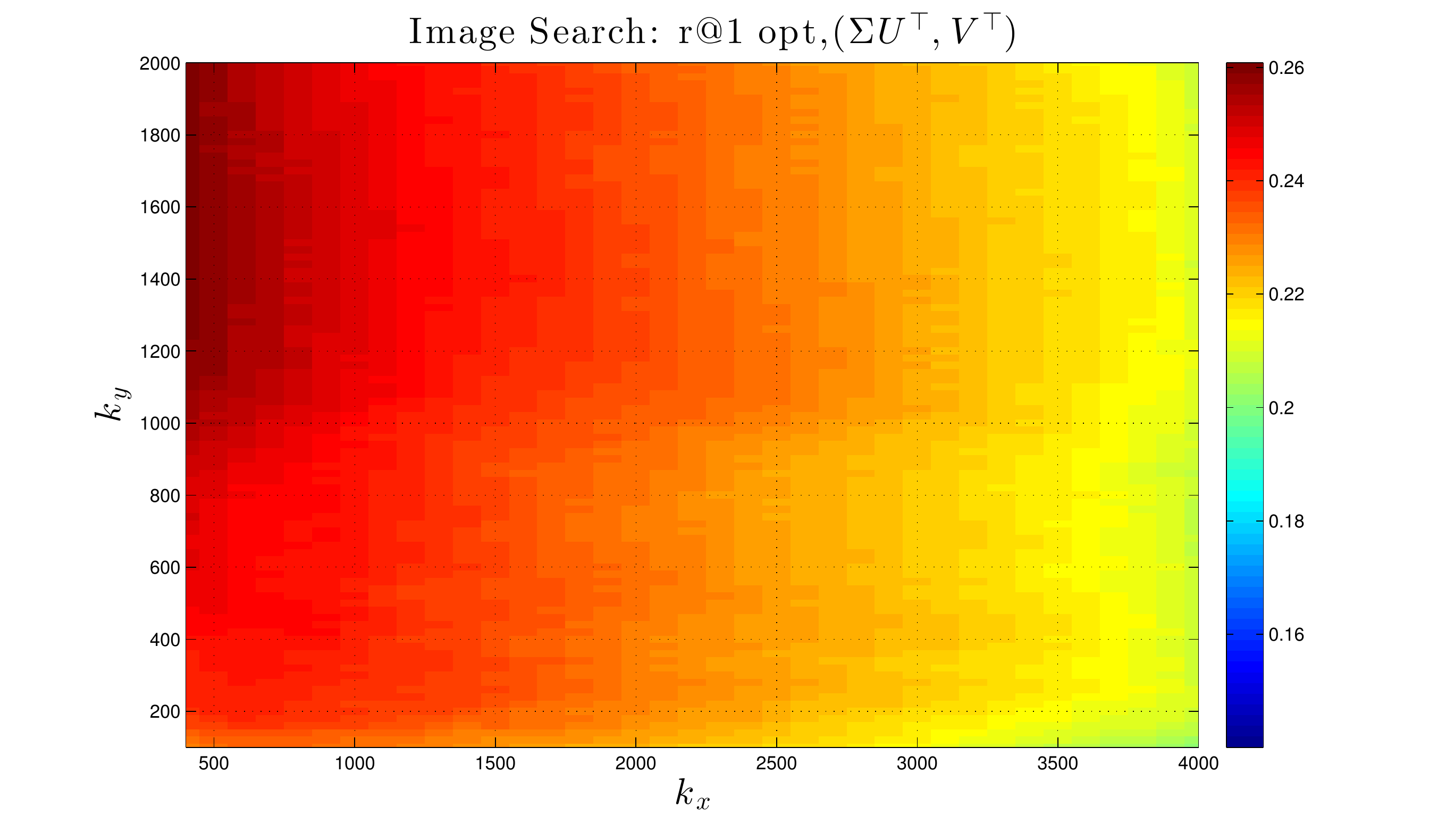}
     %  \centering
         \caption{Tikhonov Cross Validation \\$(\gamma_x=\sigma^2_{k_x}, \gamma_y=\sigma^2_{k_y})$.}
    \end{subfigure} 
         \caption{Regularization Path for T-SVD CCA , and Tikhonov CCA on bidirectional retrieval on Flickr30K with VGG features ($4096$ dimensions) for the image and  HKSE(rb,rbf) ($2000$ dimensions). Cross validation was performed on the validation set  on grid going from $400$ to $4000$ with step size of $100$ for $k_x$, and from $200$ to $2000$ with a step size of $20$ for $k_{y}$ . We report r@1 of the retrieved query over the validation set (Higher is better, in red). We see that T-SVD and Tikhonov select the same region of interest, justifying the T-SVD guided Tikhonov approach.}
  \end{figure*}

  \begin{figure}[ht]
%\hspace*{-0.8in}
  \begin{subfigure}[t]{0.6\textwidth}        
        \includegraphics[height=2.2in]{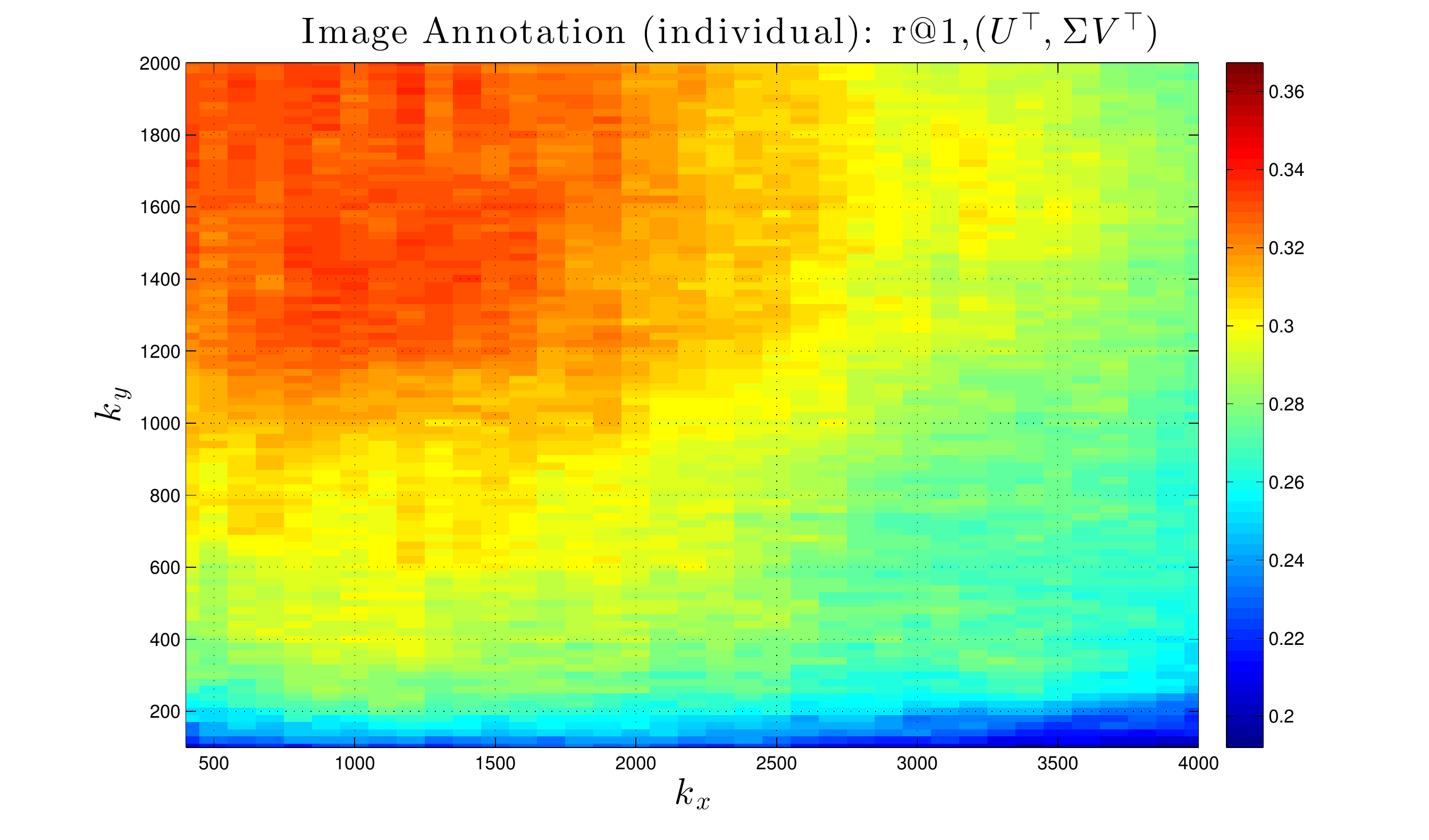}
         \caption{T-SVD Cross Validation.} %\\  \textcolor{red}{TO BE REPLACED HERE}}
    \end{subfigure}
~
\hspace*{-0.6in}
    \begin{subfigure}[t]{0.6\textwidth} 
            \includegraphics[height=2.2in]{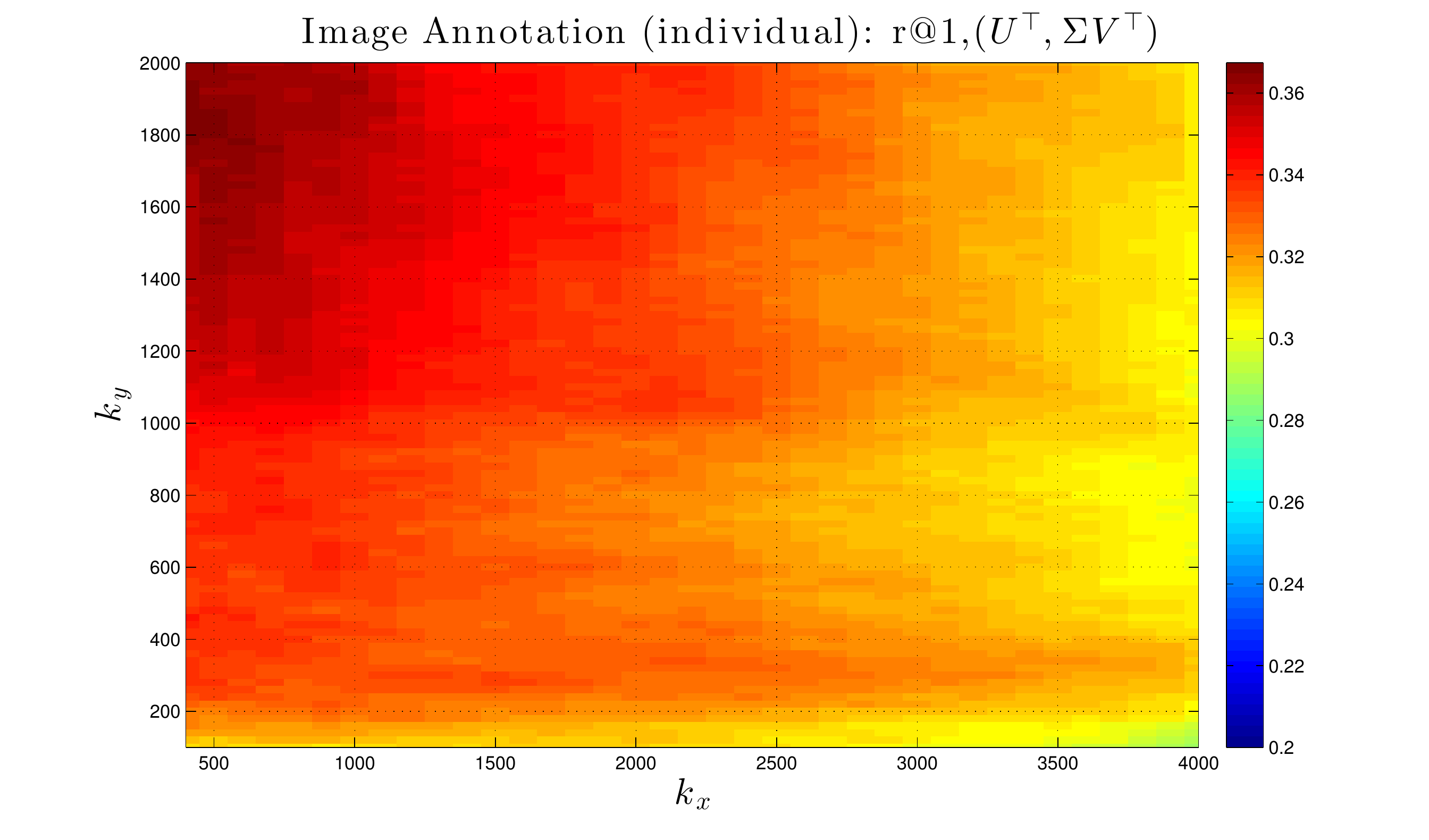}
             \caption{Tikhonov Cross Validation \\$(\gamma_x=\sigma^2_{k_x}, \gamma_y=\sigma^2_{k_y})$.}.% \\\textcolor{red}{TO BE REPLACED HERE}}
    \end{subfigure}

         \caption{  Regularization Path for T-SVD CCA , and Tikhonov CCA on bidirectional retrieval on Flickr30K with VGG features ($4096$ dimensions) for the image and  HKSE(rb,rbf) ($2000$ dimensions). Cross validation was performed on the validation set  on grid going from $400$ to $4000$ with step size of $100$ for $k_x$, and from $200$ to $2000$ with a step size of $20$ for $k_{y}$ . We report r@1 of the retrieved query over the validation set (Higher is better, in red). We see that T-SVD and Tikhonov select the same region of interest, justifying the T-SVD guided Tikhonov approach.}
  \end{figure} 

%Annotation: predicting the pack of 5 captions: 
 \begin{figure}[t!]
%\hspace*{-0.8in}
   \begin{subfigure}[t]{0.6\textwidth}        
        \includegraphics[height=2.2in]{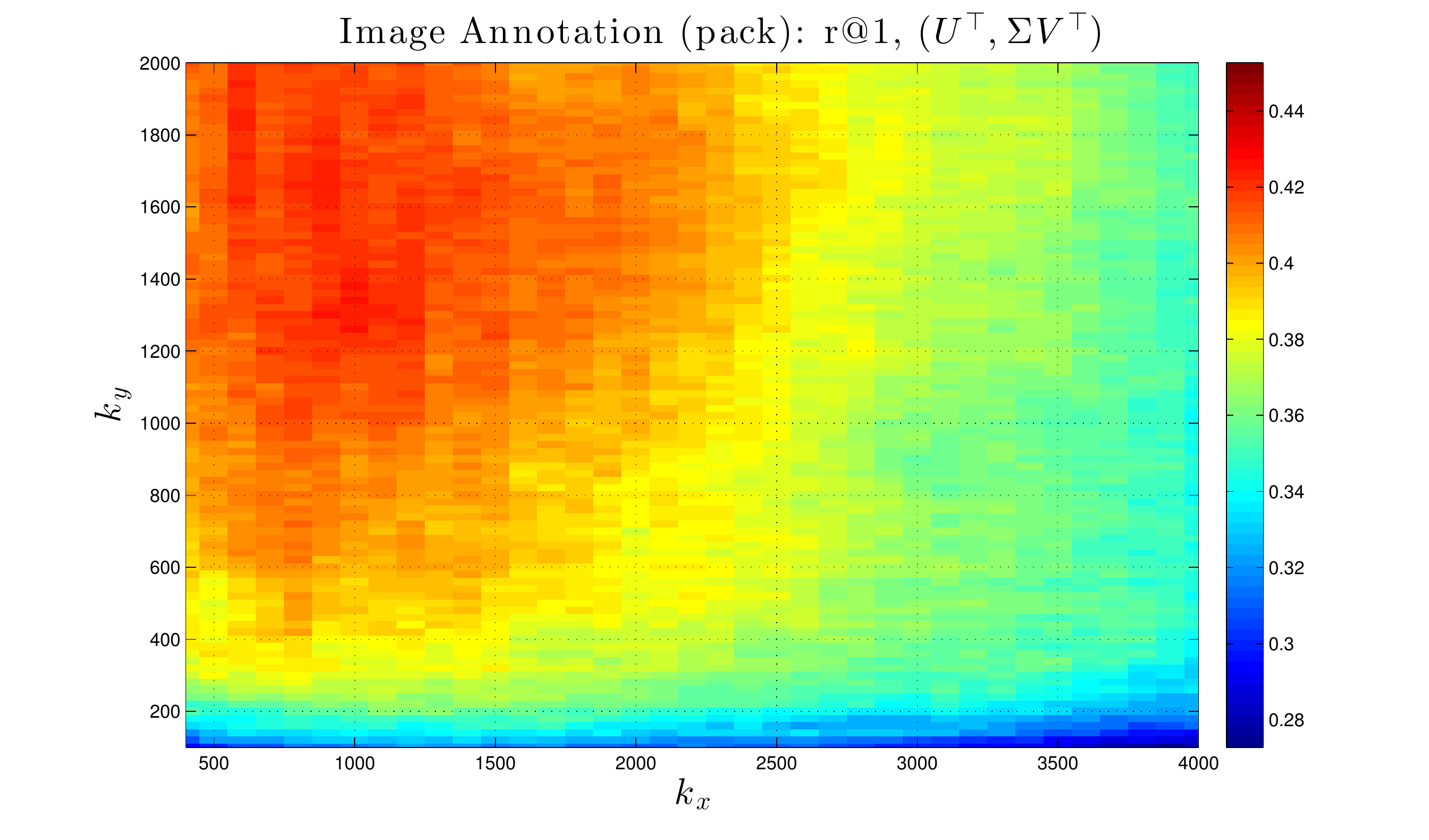}
         \caption{T-SVD Cross Validation}
    \end{subfigure}
~    \hspace*{-0.6in}
    \begin{subfigure}[t]{0.6\textwidth} 

            \includegraphics[height=2.2in]{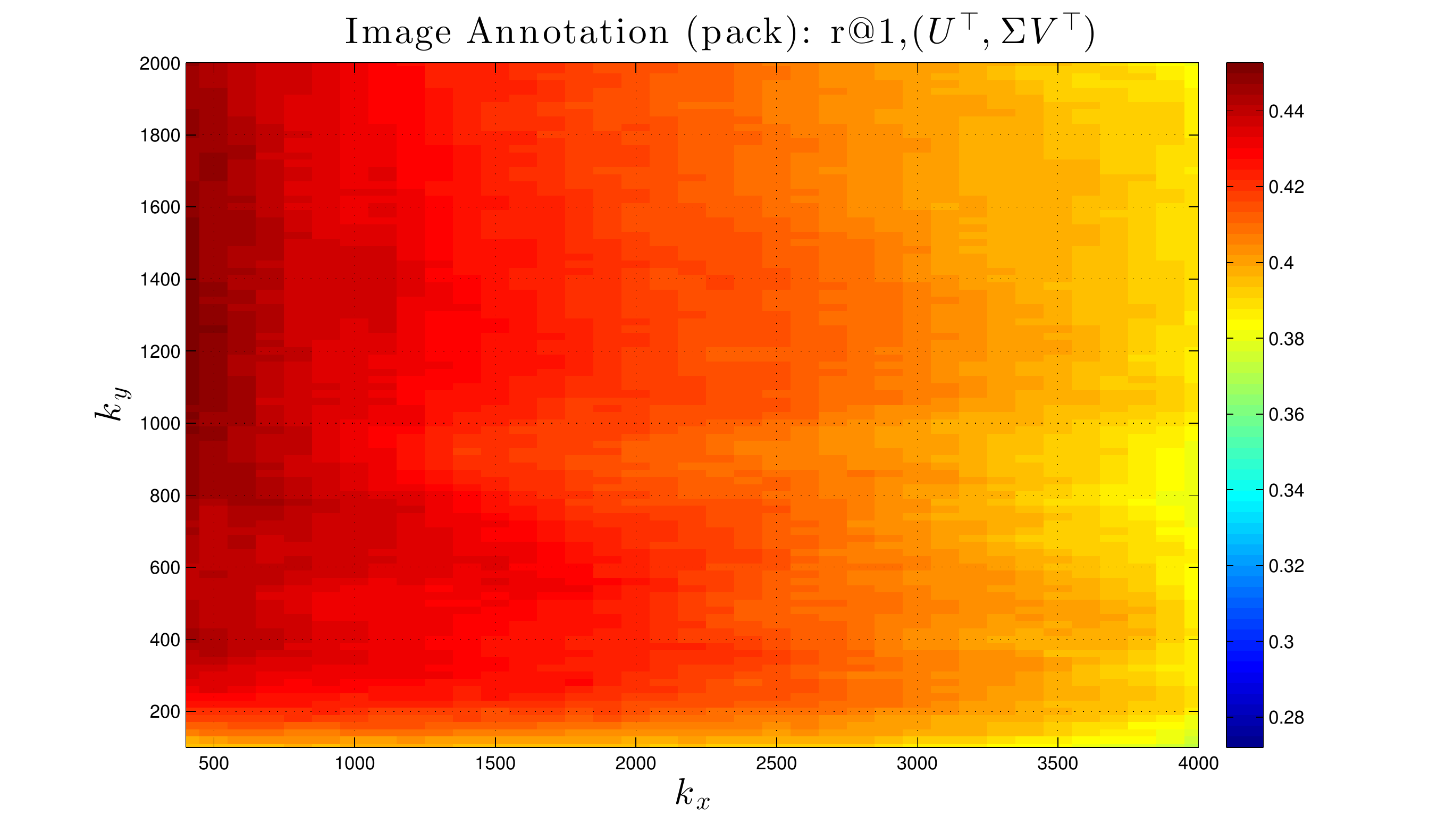}
             \caption{Tikhonov Cross Validation \\$(\gamma_x=\sigma^2_{k_x}, \gamma_y=\sigma^2_{k_y})$.}
    \end{subfigure}
         \caption{   Regularization Path for T-SVD CCA , and Tikhonov CCA on bidirectional retrieval on Flickr30K with VGG features ($4096$ dimensions) for the image and  HKSE(rb,rbf) ($2000$ dimensions). Cross validation was performed on the validation set  on grid going from $400$ to $4000$ with step size of $100$ for $k_x$, and from $200$ to $2000$ with a step size of $20$ for $k_{y}$ . We report r@1 of the retrieved query over the validation set (Higher is better, in red). We see that T-SVD and Tikhonov select the same region of interest, justifying the T-SVD guided Tikhonov approach.}
 \label{fig:packof5}
  \end{figure} 
\section{Examples of Annotation and Search with AW-CCA with HKSE(rbf,rbf) and VGG }
  \begin{figure}[ht]
\hspace*{0.3in}
  \begin{subfigure}[t]{\textwidth}     
       \centering   
        \includegraphics[scale=0.5]{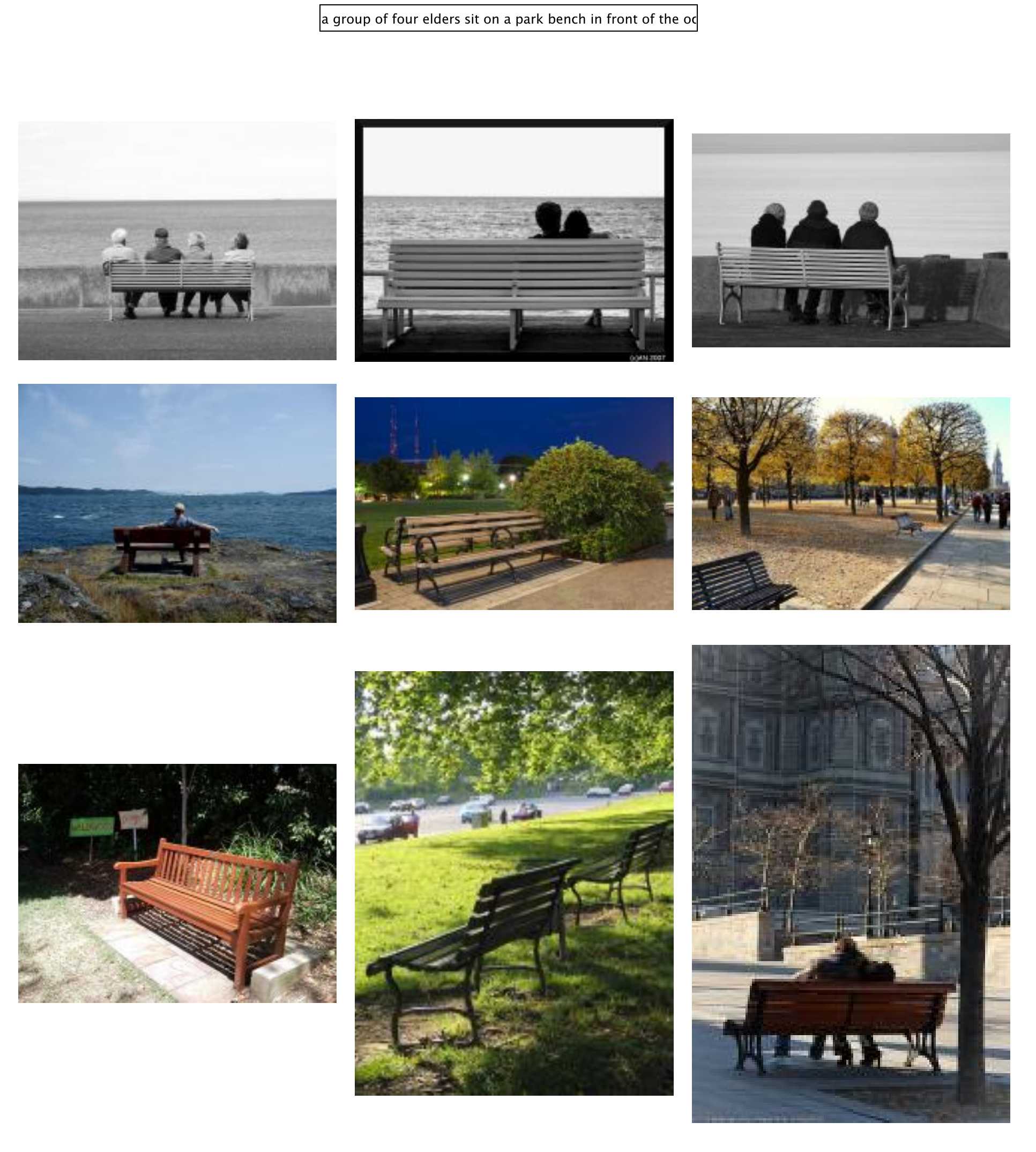}
        % \caption{T-SVD Cross Validation.} %\\  \textcolor{red}{TO BE REPLACED HERE}}
    \end{subfigure}
~
\hspace*{0.3in}
    \begin{subfigure}[t]{\textwidth} 
       \centering
            \includegraphics[scale=0.5]{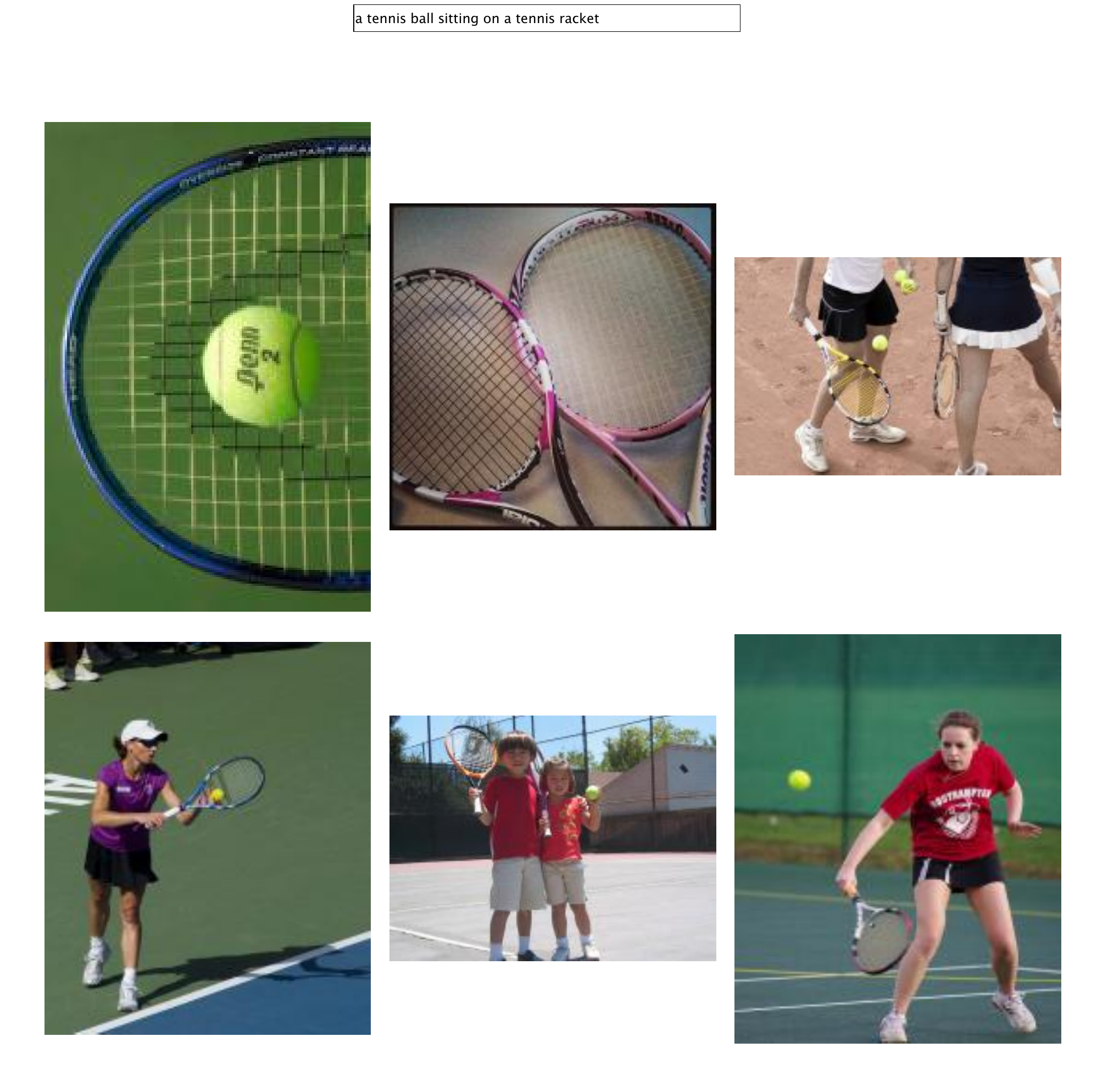}
         
           %  \caption{Tikhonov Cross Validation \\$(\gamma_x=\sigma^2_{k_x}, \gamma_y=\sigma^2_{k_y})$.}.% \\\textcolor{red}{TO BE REPLACED HERE}}
    \end{subfigure}

         \caption{ Image Search Results for two random text queries from COCO test set.}
  \end{figure} 
    \begin{figure}[t!]
\hspace*{0.3in}
  \begin{subfigure}[t]{\textwidth}     
       \centering   
        \includegraphics[scale=0.5]{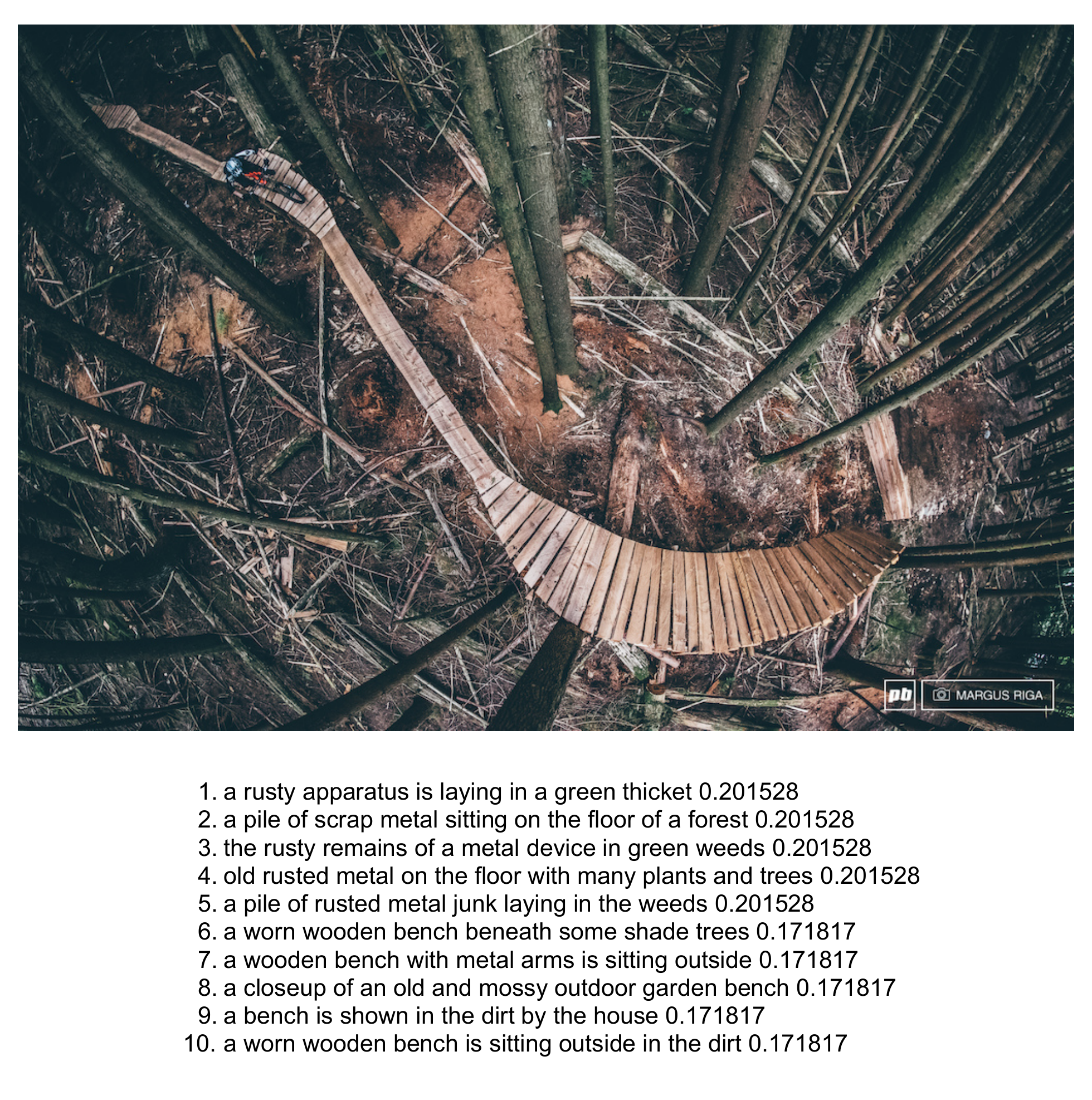}
        % \caption{T-SVD Cross Validation.} %\\  \textcolor{red}{TO BE REPLACED HERE}}
    \end{subfigure}
~
\hspace*{0.3in}
    \begin{subfigure}[t]{\textwidth} 
       \centering
            \includegraphics[scale=0.5]{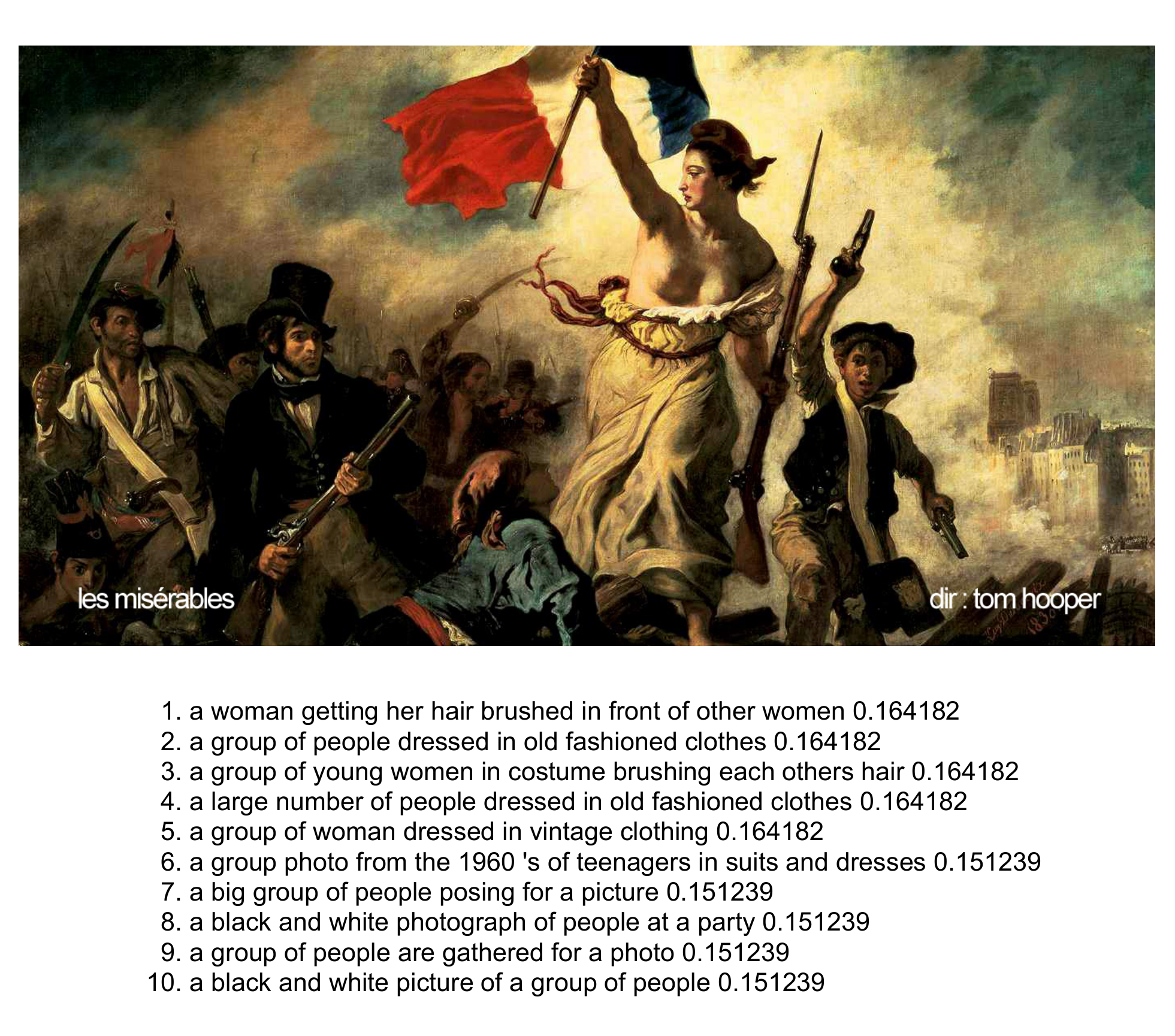}
         
           %  \caption{Tikhonov Cross Validation \\$(\gamma_x=\sigma^2_{k_x}, \gamma_y=\sigma^2_{k_y})$.}.% \\\textcolor{red}{TO BE REPLACED HERE}}
    \end{subfigure}

         \caption{ Image Annotation results for two images outside the COCO set using Ann pack.}
       
  \end{figure} 
  
\end{appendix}

\end{document}